\documentclass{article}
\usepackage{Style/arxiv}
\usepackage{natbib}
\usepackage[utf8]{inputenc}
\usepackage[T1]{fontenc}
\usepackage{hyperref}
\usepackage{url}
\usepackage{booktabs}
\usepackage{microtype}
\usepackage{graphicx}
\usepackage{doi}
\usepackage{amsmath}
\usepackage{amsthm}
\usepackage{amssymb}
\newtheorem{theorem}{Theorem}
\newtheorem{lemma}[theorem]{Lemma}
\usepackage{float}
\usepackage{caption}
\newcommand{\diff}{\mathop{}\!\text{d}}
\usepackage{placeins}
\usepackage{tikz}
\usetikzlibrary{arrows.meta}
\definecolor{ikfadred}{HTML}{8C2F1E}

\title{Low-Rank Friction for Memory-Efficient Transformer Pretraining}

\author{
Rajit Rajpal\\
	School of Mathematics\\
	University of Edinburgh\\
	Edinburgh, UK EH9 3FD \\
    \texttt{rajpal106@berkeley.edu}
	\And
Benedict Leimkuhler        \\
	School of Mathematics\\
	University of Edinburgh\\
	Edinburgh, UK EH9 3FD
}

\renewcommand{\shorttitle}{Rank-1 iKFAD}

\hypersetup{
pdftitle={Rank-1 iKFAD},
pdfauthor={Rajit Rajpal, Benedict Leimkuhler},
}

\begin{document}
\maketitle
\begin{abstract}
iKFAD is a recently proposed optimiser that replaces adaptive learning rates with adaptive friction in the momentum dynamics, yet performs as well as Adam. Its limitation is that the full friction tensor $\xi\in\mathbb{R}^{m\times n}$ carries the same $\mathcal{O}(mn)$ memory overhead per layer as Adam's second-moment buffer. Here we replace iKFAD's friction tensor $\xi$ with a rank-1 outer-product factorisation built from row and column momentum statistics, resulting in Rank-1 iKFAD (R-iKFAD). This reduces the friction memory footprint from $\mathcal{O}(mn)$ to $\mathcal{O}(m+n)$ per layer, which approximately halves iKFAD's total optimiser state. Despite this reduction, R-iKFAD maintains parity in performance with iKFAD: experiments on GPT2-Nano, TinyViT, DistilBERT and GPT2-S confirm that it matches or exceeds iKFAD while nearly halving the memory footprint and remaining comparably robust to hyperparameters. We analyse the continuous-time dynamics in two damping regimes. For linear damping ($\gamma>0$) we prove exponential convergence under strong convexity. For $\gamma=0$, the preferred option in our experiments, the friction is generated entirely from past momentum and switches off as the momentum vanishes, so geometric convergence cannot be shown. We nonetheless prove convergence to the minimiser, together with matching upper and lower bounds on the energy: of order $t^{-1}$ when the regularisation scale $\epsilon_{\mathrm{stab}}$ is zero, and of order $t^{-1/2}$ when it is positive. To our knowledge this is the first convergence rate for a rank-1 factored optimiser in continuous time, and the first such result that does not require positive damping.
\end{abstract}

\section{Introduction}

Adam~\citep{KiBa2015} remains the workhorse optimizer of modern deep learning, combining momentum with coordinate-wise second-moment scaling to accelerate convergence on complex loss landscapes. Adam evolves two auxilliary parameters: the first and second moment estimates. Assuming the model has $N$ parameters, the auxilliary states occupy $2N$ total memory cost beyond the $N$ parameters themselves. At large scale, this $2N$ overhead presents a significant hardware bottleneck for large-model training. To mitigate this, Adafactor~\citep{ShSt2018} introduces a rank-1 factorisation of the second-moment matrix. For a matrix weight $W \in \mathbb{R}^{m \times n}$, Adafactor replaces Adam's second-moment matrix $V \in \mathbb{R}^{m \times n}$ with a factored approximation $\hat{V} = R C^T/(\mathbf{1}_m^T R)$,  where $R \in \mathbb{R}^m$ and $C \in \mathbb{R}^n$ track the row-wise and column-wise squared gradients, respectively. By eliminating the first-moment buffer, Adafactor reduces the per-layer auxiliary state complexity from $\mathcal{O}(mn)$ to $\mathcal{O}(m + n)$ (achieving memory savings of over 99\% when large weight matrices are involved) while maintaining competitive transformer pre-training performance.

\noindent A useful approach to analyzing optimizer dynamics is based on continuous-time ordinary differential equations (ODEs)~\citep{DaGa2020, Ka2024}, where updates are defined as discrete time-stepping approximations. For parameter state vector $x$, momentum vector $p$, and second-moment vector $\zeta$, continuous-time Adam is given in Eqs. (\ref{eq:ct_ad_1})-(\ref{eq:ct_ad_3}). While Adam incorporates parameter-wise adaptation into the \emph{position} equation ($\dot{x}$) via $1/\sqrt{\zeta}$, \citet{KaRaLe+2026} demonstrated that coordinate-wise adaptation can alternatively be integrated into the \emph{momentum} equation ($\dot{p}$). This is achieved by replacing the constant damping scalar factor $\gamma$ with a dynamic, coordinate-wise friction state $\xi$, resulting in the Individual Kinetic Friction-Adaptive Descent (iKFAD) dynamics (\ref{eq:ikfad_1})-(\ref{eq:ikfad_3})

\begin{minipage}[t]{0.48\textwidth}
\begin{align}
    \dot{x} &= \frac{p}{\sqrt{\zeta} + \epsilon}, \label{eq:ct_ad_1}\\
    \dot{p} &= -\nabla f(x) - \gamma p,\\
    \dot{\zeta} &= [\nabla f(x)]^2 - \alpha \zeta. \label{eq:ct_ad_3}
\end{align}
\end{minipage}\hfill
\begin{minipage}[t]{0.48\textwidth}
\begin{align}
    \dot{x} &= p, \label{eq:ikfad_1}\\
    \dot{p} &= -\nabla f(x) - (\gamma + \xi) \odot p,\\
    \dot{\xi} &= \frac{[p]^2}{\mu} - \alpha\, \xi. \label{eq:ikfad_3}
\end{align}
\end{minipage}

\vspace{0.5em}
\noindent These dynamics require storage of a full adaptive friction tensor $\xi$ alongside the momentum $p$. They incur an auxiliary footprint of $2N$ scalars, identical to Adam's auxiliary storage. Originally introduced for scalar friction~\citep{KaLeSt2023},iKFAD allows friction-based adaptation to individual parameters while matching Adam's performance on standard vision and NLP benchmarks. This suggests an apparent gap: Adafactor obtains its saving by factoring the object Adam stores in the position equation, yet no available counterpart exists for the object that iKFAD stores in the momentum equation. Figure~\ref{fig:quadrant} arranges the four methods discussed in this paper. The lower-right cell is the new method presented here:  Rank-1 iKFAD (R-iKFAD) applies Adafactor's factorisation to the friction tensor, reducing iKFAD's auxiliary state from $2N$ to $N + \sum_\ell (m_\ell + n_\ell)$ scalars while retaining full momentum.

\section{Methodology: Rank-1 iKFAD}\label{sec:method}

\noindent Adafactor compresses Adam's second moment while Rank-1 iKFAD (R-iKFAD) compresses iKFAD's friction. Writing both as continuous-time systems for a matrix weight $X \in \mathbb{R}^{m\times n}$ clarifies this correspondence, in the same sense that the systems above approximate Adam~\citep{DaGa2020} and iKFAD~\citep{Ka2024}. Adafactor implementations generally omit momentum.  We retain this momentum $P$ and refer to the resulting scheme as \textbf{Adafactor-m} (left).   (Retaining momentum is deliberate as Adafactor is known to suffer training instabilities in certain transformer regimes when momentum is discarded entirely~\citep{ShSt2018}.)   Both continuous systems use a constant second-moment linear decay coefficient $\alpha$ to be directly comparable. The discrete Adafactor-m baseline follows the schedule $\beta_{2,t}=1-t^{-0.8}$ (Appendix~\ref{app:hyperparams}). R-iKFAD (right) maintains the same momentum component and factors the friction matrix instead:

\noindent
\begin{minipage}[t]{0.49\textwidth}
\begin{align}
    \dot{X} &= \frac{P}{\sqrt{\hat{\zeta}} + \epsilon}, \\[1pt]
    \dot{P} &= -\nabla f(X) - \gamma P, \\[1pt]
    \dot{R} &= [\nabla f(X)]^2\, \mathbf{1}_n - \alpha R, \\[1pt]
    \dot{C} &= ([\nabla f(X)]^2)^\top \mathbf{1}_m - \alpha C.
\end{align}
\end{minipage}\hfill
\begin{minipage}[t]{0.47\textwidth}
\begin{align}
\dot{X} &= P, \label{eq:rikfad_x}\\
\dot{P} &= -\nabla f(X) - \tilde{\xi} \odot P - \gamma P,\\
\dot{R} &= \frac{1}{\mu} [P]^2 \mathbf{1}_n - \alpha R,\\
\dot{C} &= \frac{1}{\mu} ([P]^2)^\top \mathbf{1}_m - \alpha C. \label{eq:rikfad_C}
\end{align}
\end{minipage}

\vspace{0.4em}
\noindent Here $\odot$ is the element-wise product, $[P]^2_{ij}=P_{ij}^2$, and the divisions in the position equations are element-wise. Both systems carry the same linear damping $\gamma$. Each builds a rank-1 object from its factors,
\begin{equation}
\hat{\zeta} = \frac{RC^\top}{\mathbf{1}_m^\top R},
\qquad
\tilde{\xi}_{ij} = \frac{R_i C_j}{\mathbf{1}_m^\top R + \epsilon_{\mathrm{stab}}},
\label{eq:xi_def}
\end{equation}
with $\epsilon,\epsilon_{\mathrm{stab}}>0$ small regularization coefficients, the latter ensuring numerical stability near $R=0$. 

Adafactor-m is thus the Adam system with the full second moment $\zeta$ replaced by the rank-1 estimate $\hat{\zeta}$, and R-iKFAD is the iKFAD system with the full friction tensor $\xi$ replaced by the rank-1 estimate $\tilde{\xi}$. This is illustrated more concisely in Figure~\ref{fig:quadrant}. The two methods differ in the same way Adam differs from iKFAD: Adafactor-m's factors aggregate the squared \emph{gradients} and act in the position equation, whereas R-iKFAD's aggregate the squared \emph{momenta} and act as damping in the momentum equation. In both cases the adapted state per layer (the second moment or the friction) falls from $\mathcal{O}(mn)$ to $\mathcal{O}(m+n)$, while full momentum evolution is retained. We emphasize that R-iKFAD is designed to maintain parity with iKFAD (which matches Adam~\citep{KaRaLe+2026} in performance), rather than outperforming it in terms of convergence.

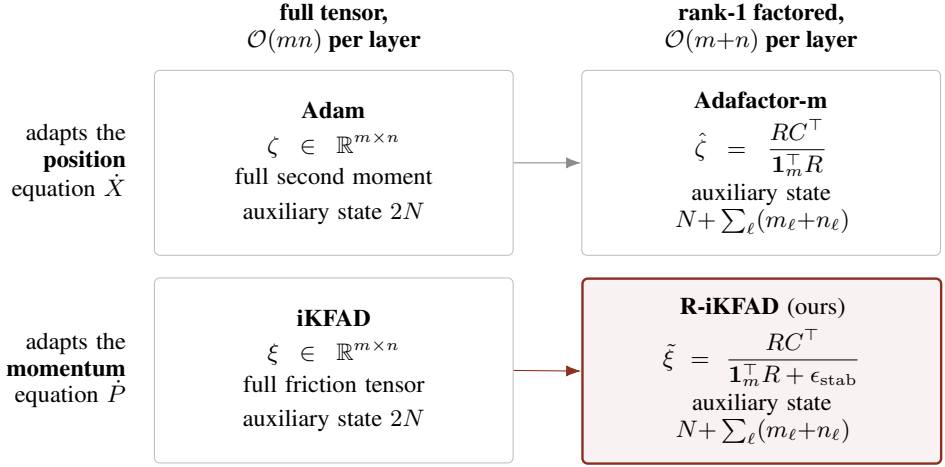
\begin{figure}[t]
\centering
\begin{tikzpicture}[font=\small,
  cell/.style={draw=black!25, rounded corners=2pt, text width=4.25cm,
               minimum height=2.45cm, align=center, inner sep=7pt,
               anchor=north west},
  star/.style={cell, draw=ikfadred, line width=0.9pt, fill=ikfadred!6},
  colh/.style={font=\small\bfseries, align=center, anchor=south},
  rowh/.style={font=\small, align=right, text width=2.2cm, anchor=east}]

  \def\cw{5.65}\def\ch{2.75}

  \node[cell] (adam) at (0,0) {\textbf{Adam}\\[3pt]
      $\zeta \in \mathbb{R}^{m \times n}$\\[2pt]
      {\footnotesize full second moment}\\[3pt]
      {\footnotesize auxiliary state $2N$}};
  \node[cell] (adaf) at (\cw,0) {\textbf{Adafactor-m}\\[3pt]
      $\hat{\zeta} = \dfrac{R C^{\top}}{\mathbf{1}_m^{\top} R}$\\[2pt]
      {\footnotesize auxiliary state}\\ {\footnotesize $N{+}\sum_\ell(m_\ell{+}n_\ell)$}};
  \node[cell] (ikfad) at (0,-\ch) {\textbf{iKFAD}\\[3pt]
      $\xi \in \mathbb{R}^{m \times n}$\\[2pt]
      {\footnotesize full friction tensor}\\[3pt]
      {\footnotesize auxiliary state $2N$}};
  \node[star] (rik) at (\cw,-\ch) {\textbf{R-iKFAD} (ours)\\[3pt]
      $\tilde{\xi} = \dfrac{R C^{\top}}{\mathbf{1}_m^{\top} R + \epsilon_{\mathrm{stab}}}$\\[2pt]
      {\footnotesize auxiliary state}\\ {\footnotesize $N{+}\sum_\ell(m_\ell{+}n_\ell)$}};

  \node[colh] at ([yshift=3pt]adam.north) {full tensor,\\$\mathcal{O}(mn)$ per layer};
  \node[colh] at ([yshift=3pt]adaf.north) {rank-1 factored,\\$\mathcal{O}(m{+}n)$ per layer};

  \node[rowh] at ([xshift=-7pt]adam.west) {adapts the\\\textbf{position}\\equation $\dot{X}$};
  \node[rowh] at ([xshift=-7pt]ikfad.west) {adapts the\\\textbf{momentum}\\equation $\dot{P}$};

  \draw[-{Latex[length=2mm]}, black!45] (adam.east) -- (adaf.west);
  \draw[-{Latex[length=2mm]}, ikfadred] (ikfad.east) -- (rik.west);
\end{tikzpicture}
\caption{The same rank-1 compression applied in two different equations. Adafactor-m factorizes Adam's second moment $\zeta$ which scales the position update while R-iKFAD factorizes iKFAD's friction $\xi$ which damps the momentum update. Both reduce the adapted state per layer from $\mathcal{O}(mn)$ to $\mathcal{O}(m+n)$ while keeping full momentum, so total auxiliary memory falls from $2N$ to roughly $N$.}
\label{fig:quadrant}
\end{figure}

\paragraph{Splitting Discretization.}
To construct a stable discrete-time algorithm, we turn to principles of geometric numerical integration of Hamiltonian systems~\citep{LeRe2005}, which have proved a robust foundation for discretization of dynamics-based optimisation schemes. Following~\citet{KaRaLe+2026}, we decompose the vector field of \eqref{eq:rikfad_x}--\eqref{eq:rikfad_C} into four components, three of these being directly integrable.  The fourth is easily approximated  by an additional splitting step.   Specifically we write:

\begin{equation}
\begin{pmatrix} \dot{X} \\ \dot{P} \\ \dot{R} \\ \dot{C} \end{pmatrix}
=
\underbrace{\begin{pmatrix} P \\ 0 \\ 0 \\ 0 \end{pmatrix}}_{\mathrm{A}}
+
\underbrace{\begin{pmatrix} 0 \\ -\nabla f(X) \\ 0 \\ 0 \end{pmatrix}}_{\mathrm{B}}
+
\underbrace{\begin{pmatrix}
  0 \\
  -\dfrac{R C^\top}{\mathbf{1}_m^\top R + \epsilon_{\mathrm{stab}}} \odot P \\[6pt]
  \dfrac{1}{\mu}[P]^2\mathbf{1}_n - \alpha R \\[6pt]
  \dfrac{1}{\mu}([P]^2)^\top\mathbf{1}_m - \alpha C
\end{pmatrix}}_{\mathrm{C}}
+
\underbrace{\begin{pmatrix} 0 \\ -\gamma P \\ 0 \\ 0 \end{pmatrix}}_{\mathrm{D}}.
\end{equation}
Sub-flows A, B, and D can be integrated analytically: A gives position drift, B applies the gradient kick, D applies linear friction (becoming the identity when $\gamma=0$, i.e., R-iKFAD0). For step C, we damp the momentum for half a step using the current factors, update the factors for a full step with momentum held fixed, then damp the momentum for another half step using the updated factors.
The composition applied in practice is a first-order Lie--Trotter splitting:
\[
  \Phi_h = \Phi_h^{\mathrm{D}} \circ \Phi_h^{\mathrm{C}} \circ
           \Phi_h^{\mathrm{A}} \circ \Phi_h^{\mathrm{B}},
\]
where sub-steps are executed in the sequence $\mathrm{B} \to \mathrm{A} \to \mathrm{C} \to \mathrm{D}$ (BACD)~\citep{KaRaLe+2026}. 
Position updates use the gradient-kicked momentum $P^{(B)}$ rather than $P_n$. Table~\ref{tab:splitting_updates} in Appendix~\ref{app:splitting} gives the explicit analytical update rules for each sub-step across the complete BACD sequence over time step $h$, together with the treatment of the coupling inside sub-flow~C.

\paragraph{Memory Complexity.}
For a network whose $\ell$th parameter matrix is $W_\ell \in \mathbb{R}^{m_\ell \times n_\ell}$, let $N = \sum_\ell m_\ell n_\ell$ denote the total number of scalar model parameters. Figure~\ref{fig:quadrant} breaks down the auxiliary optimizer memory footprint required on top of the model weights. For non-2D parameters (biases, LayerNorm scales) that are not rank-1 factorizable, R-iKFAD reduces to element-wise iKFAD, contributing a negligible additional scalar per parameter. R-iKFAD requires approximately $N$ auxiliary scalars in total, effectively halving the optimizer auxiliary memory footprint relative to both iKFAD and Adam.   Adafactor-m has a similar auxiliary memory requirement ${\approx}N$
 and is therefore the primary baseline in our experiments.

\FloatBarrier
\section{Theoretical Results}

We next discuss our analysis of the continuous-time system \eqref{eq:rikfad_x}--\eqref{eq:rikfad_C} in both damping regimes. This section states the results and reports the numerical verification. The proofs are in Appendix~\ref{app:proofs}. 

For $\gamma>0$, a Lyapunov argument gives exponential convergence under strong convexity.

\begin{theorem}[Exponential convergence for $\gamma>0$]\label{thm:rk1_exp}
Suppose that $f$ is $m_f$-strongly convex and has an
$L_f$-Lipschitz gradient. Let
$\gamma,\alpha,\mu,\epsilon_{\mathrm{stab}}>0$, and let $R(0)$ and
$C(0)$ be componentwise nonnegative. Then there are constants
$M,\rho>0$ such that
\begin{equation}
 f(X(t))-f(X^*)+\|X(t)-X^*\|_F^2+\|P(t)\|_F^2
 +\|R(t)\|_2^2+\|C(t)\|_2^2
 \leq M e^{-\rho t}
\label{eq:positive_gamma_result}
\end{equation}
for every $t\geq0$.
\end{theorem}

In our experiments the $\gamma=0$ and $\gamma>0$ variants performed very similarly, and R-iKFAD0 is preferable since it has one fewer hyperparameter to optimize, which makes the sweep easier. With $\gamma=0$ the overall adaptive friction remains non-negative but can become small with the momentum. Proving convergence in this setting is challenging, but we have succeeded in doing so under strong convexity.  

\begin{theorem}[Convergence when $\gamma=0$]\label{thm:g0_conv}
Suppose that $f$ is strongly convex and has a locally Lipschitz gradient. Let $\alpha,\mu,\epsilon_{\mathrm{stab}}>0$, and suppose that $R(0)$ and $C(0)$ are componentwise nonnegative. Then every solution of \eqref{eq:rikfad_x}--\eqref{eq:rikfad_C} with $\gamma=0$ converges to
\[
(X,P,R,C)=(X_\star,0,0,0),
\]
where $X_\star$ is the unique minimizer of $f$.
\end{theorem}

The decay at $\gamma=0$ is algebraic rather than exponential, meaning that the error falls like $t$ to some power. The power depends on whether or not a stabilizing parameter $\epsilon_{\mathrm{stab}}$ is included.  Both this theorem and the numerical evidence that follows are stated in terms of the energy
\begin{equation}
\mathcal{E}(t)=f(X(t))-f(X^*)+\tfrac12\lVert P(t)\rVert_F^2,
\end{equation}
the sum of the objective gap and the kinetic energy term, which is nonnegative and vanishes exactly at $(X^*,0)$.

\begin{theorem}[Algebraic rates]\label{thm:rates}
Suppose $f$ is strongly convex, its gradient is locally Lipschitz,
and it is $C^2$ in a neighbourhood of $X^*$. Let
$\gamma=0$ and $\alpha,\mu>0$, and let the initial factors be component-wise nonnegative.

If $\epsilon_{\mathrm{stab}}>0$ and $\mathcal{E}(0)>0$, there are $a,A>0$
such that
\[
 \frac a{\sqrt{1+t}}\le \mathcal{E}(t)\le\frac A{\sqrt{1+t}},
 \qquad t\ge0.
\]
If $\epsilon_{\mathrm{stab}}=0$, assume in addition that
$R(0)=C(0)=0$, and use the convention $\tilde{\xi}=0$ at $\mathbf{1}_m^\top R=0$.
If $\mathcal{E}(0)>0$, there are $b,B>0$ such that
\[
 \frac b{1+t}\le \mathcal{E}(t)\le\frac B{1+t},
 \qquad t\ge0.
\]
\end{theorem}

The two regimes are separated by a crossover when $\mathcal{E}$ is of order $\alpha\mu\epsilon_{\mathrm{stab}}$. Section~\ref{app:g0_formal_rate} gives a formal averaging calculation that identifies the exponents, while Appendix~\ref{app:g0_rates} proves those results, and Section~\ref{app:g0_numerics} verifies them numerically. The two-sided bounds establish \(t^{-1/2}\) as the asymptotic decay rate for every \(\epsilon_{\mathrm{stab}}>0\), motivating a small stabiliser to delay the crossover to this slower regime.

\subsection{Formal rate calculation}\label{app:g0_formal_rate}

We use a formal averaging calculation to estimate the decay. The exponents it produces are rigorously established in Appendix~\ref{app:g0_rates}.
Suppose that $P$ oscillates faster than the factors change. For $\epsilon_{\mathrm{stab}}=0$, start with $\sum_kR_k>0$. Replacing the oscillatory terms by the corresponding averages gives
\[
\dot{\mathcal E}\simeq-\frac{\mathcal E^2}{\alpha\mu}.
\]
The reduced equation admits the solution
\[
\mathcal E(t)\simeq\frac{\alpha\mu}{t}.
\]
When $\epsilon_{\mathrm{stab}}>0$ and the factors are small relative to $\epsilon_{\mathrm{stab}}$, the same calculation gives
\[
\dot{\mathcal E}\simeq
-\frac{\mathcal E^3}{\alpha^2\mu^2\epsilon_{\mathrm{stab}}},
\qquad
\mathcal E(t)\simeq
\alpha\mu\sqrt{\frac{\epsilon_{\mathrm{stab}}}{2t}}.
\]
The two reduced regimes meet when $\mathcal E$ is of order
$\alpha\mu\epsilon_{\mathrm{stab}}$.
This calculation relies on time-scale separation and an averaging closure. Neither assumption is needed for Theorem~\ref{thm:rates}, which obtains the same exponents without them.

\subsection{Numerical validation of proved results}\label{app:g0_numerics}

We tested the dynamics of $\gamma=0$ in a strongly convex quadratic,
$f(X)=\tfrac12 X^\top AX$ with $A=BB^\top\!/d+0.8I$ and $d=mn$, where $B$ has independent
standard normal entries, so the smallest eigenvalue of $A$ is at least $0.8$.  We monitor the
energy $\mathcal E(t)=\tfrac12 X^\top AX+\tfrac12\lVert P\rVert_F^2$. The decay for both stabilization settings may be found in Figure~\ref{fig:g0_decay}. Table~\ref{tab:g0_slopes} gives the slopes of $\log\mathcal E$ against $\log t$, obtained by
ordinary least squares on successive decades out to $T=10^{7}$ for three different matrix sizes. Both
of the powers specified by our theorem are recovered and neither apparently depends on the matrix dimension. 
The approach to
slope $-1/2$ is slower for larger $mn$, and the crossover point illustrates the reason for this: 
the second regime begins only
once $\mathcal{E}\lesssim\alpha\mu\,\epsilon_{\mathrm{stab}}$, and the constant in the
first-regime law grows with $mn$, so the crossing occurs later for larger systems. The
prefactors scale with $\alpha$, $\mu$ and $\epsilon_{\mathrm{stab}}$ just as the calculation of
Section~\ref{app:g0_formal_rate} predicts.    Appendix~\ref{app:g0_prefactors} reports the measured
ratios and Appendix~\ref{app:g0_setup} gives the integrators, tolerances, step-sizes, seeds and
initialization.

\begin{figure}[H]
    \centering
    \includegraphics[width=0.72\linewidth]{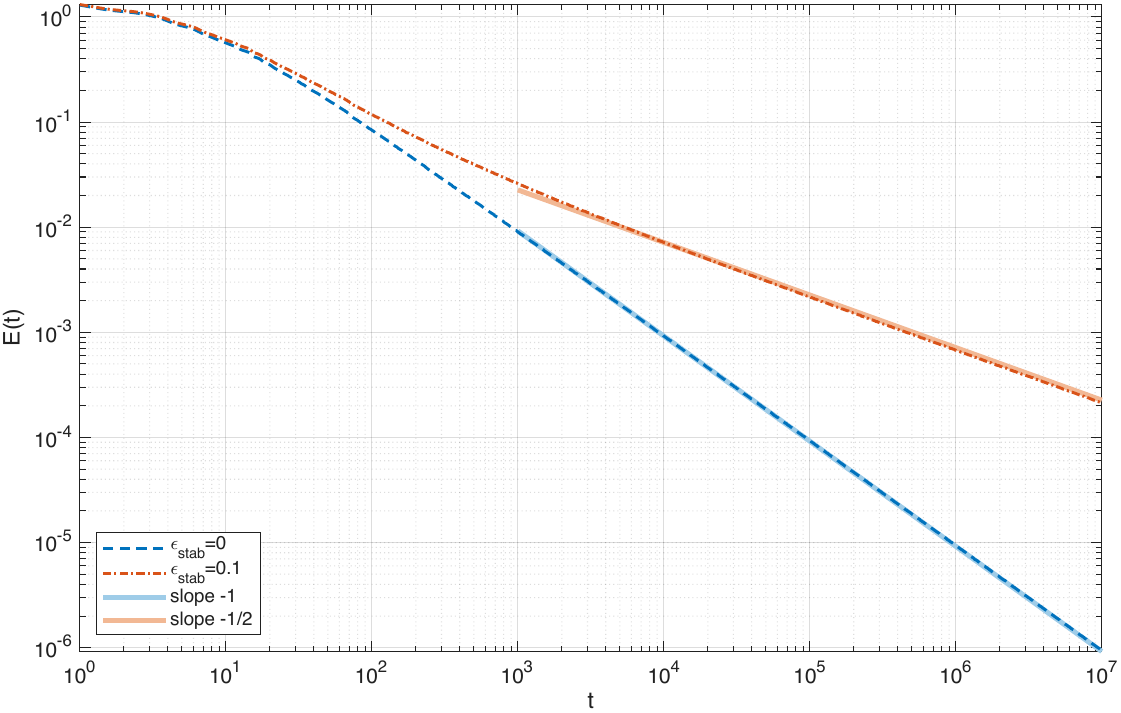}
    \caption{Energy decay from BACD splitting with $m=4$, $n=3$, $\gamma=0$, $\alpha=\mu=1$, $h=0.02$, and final time $10^7$. Both runs have $R(0)=C(0)=0$. The blue dashed curve has $\epsilon_{\mathrm{stab}}=0$, and the orange dash-dotted curve has $\epsilon_{\mathrm{stab}}=0.1$. The solid lines have slopes $-1$ and $-1/2$, respectively.}
    \label{fig:g0_decay}
\end{figure}

\begin{table}[t]
\centering
\caption{Local slopes of $\log\mathcal{E}$ against $\log t$ on successive decades, from the BACD
splitting at $h=0.02$ with $\alpha=\mu=1$, integrated to $T=10^{7}$. Both exponents are independent
of the matrix dimensions.}
\label{tab:g0_slopes}
\small
\begin{tabular}{lcccc}
\toprule
 & $[10^3,10^4]$ & $[10^4,10^5]$ & $[10^5,10^6]$ & $[10^6,10^7]$ \\
\midrule
\multicolumn{5}{l}{$\epsilon_{\mathrm{stab}}=0$ \quad (proved exponent $-1$)} \\
$m=n=2$   & $-0.955$ & $-0.980$ & $-0.986$ & $-0.989$ \\
$m=4,n=3$ & $-0.991$ & $-0.996$ & $-0.996$ & $-0.998$ \\
$m=6,n=5$ & $-0.993$ & $-0.997$ & $-0.998$ & $-0.999$ \\
\midrule
\multicolumn{5}{l}{$\epsilon_{\mathrm{stab}}=0.1$ \quad (proved exponent $-1/2$)} \\
$m=n=2$   & $-0.499$ & $-0.497$ & $-0.496$ & $-0.497$ \\
$m=4,n=3$ & $-0.558$ & $-0.518$ & $-0.505$ & $-0.501$ \\
$m=6,n=5$ & $-0.593$ & $-0.530$ & $-0.509$ & $-0.503$ \\
\bottomrule
\end{tabular}
\end{table}

\subsection{Related Work: Algebraic Rates and Factored Optimisers}

Continuous-time studies with damping such as \citet{SuBoCa2016} and \citet{AtChPe+2018} prescribed as a function of time have yielded algebraic rates. Our results are proven for a different setting given that our damping is determined by evolving factors and can vanish with them. Our damping model is also nonlinear and thus more complicated. In terms of factored schemes,  Adafactor~\citep{ShSt2018} introduced rank-1 factorization of optimizer state. \citet{HoLi2025} analyzed a discrete Adafactor method for non-convex objectives. \citet{NgChLi+2025} introduced H-Fac, which factorizes both the momentum and scaling estimators, and prove asymptotic convergence for its continuous-time dynamics with positive damping. These results do not come near providing rates for $\gamma=0$ situation. KFAD~\citep{KaLeSt2023} and iKFAD~\citep{KaRaLe+2026} were analyzed only in the
linearly damped case $\gamma>0$. Our $\gamma=0$ results carry over to both. In each,
the friction obeys the same non-negative linear filter as the factors $R_i$ and $C_j$:
a single scalar driven by $\|P\|_F^2$ in KFAD, and one entry $\xi_{ij}$ driven by
$P_{ij}^2$ for each parameter in iKFAD. Neither method has a stabilizer, so only the
$t^{-1}$ rate applies.

\section{Numerical experiments}\label{sec:experiments}

We set $\gamma=0$ for both iKFAD and R-iKFAD throughout (denoted \textbf{iKFAD0} and \textbf{R-iKFAD0} respectively), following~\citet{KaRaLe+2026}. In the ablation of Section~\ref{sec:ablation}, adding $\gamma>0$ gave no consistent improvement under our tested procedure. We compare both against Adam and Adafactor-m across four benchmarks spanning problems in image classification and language modelling at scales from 0.8M to 124M parameters:
\begin{itemize}
    \item \textbf{GPT2-Nano (0.8M)}~\citep{Ka2022} on Shakespeare character-level language modelling.
    \item \textbf{TinyViT (4M)}~\citep{WuZhPe+2022} on CIFAR-10~\citep{KrHi+2009} image classification.
    \item \textbf{DistilBERT (67M)}~\citep{SaDeCh+2020} on SST-2~\citep{SoPeWu+2013} sentiment classification (pretraining from scratch).
    \item \textbf{GPT2-S (124M)}~\citep{RaWuCh+2019} on OpenWebText pretraining (nanoGPT architecture)~\citep{Ka2022}.
\end{itemize}

\noindent To focus on the optimization dynamics, we use a constant learning rate, we disable weight decay, and we omit additional regularization throughout. We use Adafactor-m rather than the conventional momentum-free configuration because its auxiliary-memory requirement is similar to that of R-iKFAD. All hyperparameters are tuned via Optuna~\citep{AkSaYa+2019} (TPE sampler) with equal sweep budgets of approximately 80 trials per optimiser per model. GPT2-Nano and TinyViT received substantially more trials. The stabilizer is fixed at $\epsilon_{\mathrm{stab}}=10^{-16}$ in every run and is never tuned. See Appendix~\ref{app:hyperparams} for full configurations. As with iKFAD, the optimal scale of $\mu$ varies by several orders of magnitude between tasks due to differing momentum scales. We therefore recommend a wide-range initial sweep over $\mu$. Note that Adafactor-m requires tuning only two hyperparameters, whereas Adam, iKFAD0, and R-iKFAD0 all require three. The fixed trial budget hence favours Adafactor-m's results disproportionately. All benchmarks report mean $\pm$ standard deviation over 10 seeds, except GPT2-S, which uses a single seed due to computational constraints. The optimizer auxiliary state is the memory an optimizer adds on top of the model parameters. Table~\ref{tab:memory} in Appendix~\ref{app:memory} reports the absolute values on the four benchmarks.

\begin{figure}[t]
    \centering
    \includegraphics[width=\linewidth]{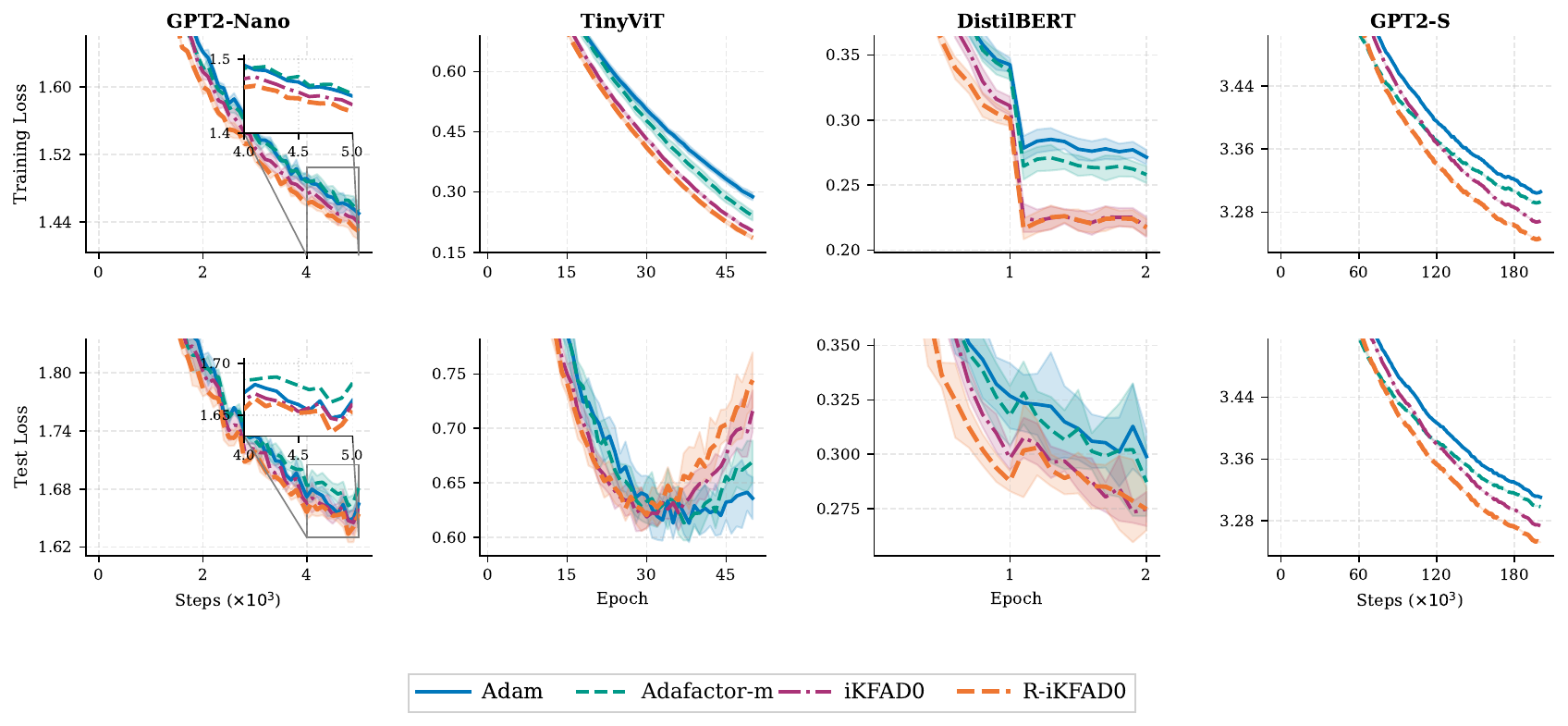}
    \caption{Loss curves for all four benchmarks, averaged over 10 random seeds (GPT2-S: a single seed). R-iKFAD0 matches or improves upon iKFAD0 on all tasks, while using approximately half the optimizer memory. Both methods are competitive with Adam and Adafactor-m, which require comparable or greater memory. Shaded regions denote one standard deviation.}
    \label{fig:losses}
\end{figure}

\subsection{Main Results}

Figure~\ref{fig:losses} shows training and test loss across all benchmarks, and Figure~\ref{fig:acc} in Appendix~\ref{app:accuracy} the corresponding accuracies. Table~\ref{tab:results} reports the test loss of the model saved at the lowest validation loss.  \citet{KaRaLe+2026} already established that iKFAD matches Adam on standard benchmarks, and Table~\ref{tab:results} reproduces that result here.  iKFAD0 equals or improves on Adam on three of the four benchmarks, the exception being TinyViT, by about one pooled standard deviation. Inspecting the loss curves for TinyViT, both iKFAD0 and R-iKFAD0 reach their minimum faster than Adam, which later attains a lower loss. 

Moreover, R-iKFAD0 matches iKFAD0 on all four benchmarks while halving the auxiliary optimizer state (Table~\ref{tab:memory}). Taken together, R-iKFAD0 performs like Adam at roughly half Adam's optimizer memory. These differences should be interpreted in light of the variability across seeds. On the three benchmarks with seed replication, iKFAD0 and R-iKFAD0 differ by less than one pooled standard deviation, whereas on TinyViT they agree to three decimal places. Where the friction methods do separate from Adam, the difference depends on the particular dataset: on DistilBERT they lead by about five pooled standard deviations, and on GPT2-S by $0.06$ nats at a single seed, whereas on the one vision task Adam leads by about one. Given how small these differences are, we do not suggest that R-iKFAD0 outperforms Adam. Rather, the statement is that the friction tensor can be rank-1 factored without compromising performance.

\begin{table}[h]
\centering
\caption{Test loss at the checkpoint selected by lowest \emph{validation} loss (mean$\pm$std over 10 seeds, GPT2-S: single seed due to computational limitations). The test set was never used for model or hyper-parameter selection. Lower is better. R-iKFAD0 matches or improves upon iKFAD0 across all benchmarks while using half the optimizer memory.}
\label{tab:results}
\small
\begin{tabular}{lrrrr}
\toprule
Dataset & Adam & Adafactor-m & iKFAD0 & R-iKFAD0 (ours) \\
\midrule
GPT2-Nano   & $1.643\pm0.008$ & $1.661\pm0.007$ & $1.640\pm0.008$ & $\mathbf{1.633\pm0.011}$ \\
TinyViT     & $\mathbf{0.593\pm0.015}$ & $0.596\pm0.010$ & $0.606\pm0.011$ & $0.606\pm0.012$ \\
DistilBERT  & $0.292\pm0.005$ & $0.280\pm0.007$ & $0.269\pm0.005$ & $\mathbf{0.268\pm0.003}$ \\
GPT2-S      & $3.301$ & $3.287$ & $3.263$ & $\mathbf{3.242}$ \\
\bottomrule
\end{tabular}
\end{table}

\subsection{Ablation: \texorpdfstring{$\gamma=0$ versus $\gamma > 0$}{gamma=0 versus gamma > 0}}\label{sec:ablation}

Figure~\ref{fig:ablation} compares R-iKFAD ($\gamma>0$, with linear damping) against R-iKFAD0 ($\gamma=0$, without linear damping) on GPT2-Nano, TinyViT, and DistilBERT (SST-2). Adding the $\gamma$ term gave no consistent improvement over $\gamma=0$ on these benchmarks under the tested procedure. This is consistent with the finding in~\citet{KaRaLe+2026} that the adaptive friction mechanism alone provides sufficient damping. We therefore recommend \textbf{R-iKFAD0} ($\gamma = 0$) as the practical default since it has one fewer hyper-parameter. For the $\gamma>0$ variant, we used the same $h$, $\alpha$, and $\mu$ settings as R-iKFAD0 given in Appendix~\ref{app:hyperparams} and additionally swept $\gamma$ log-uniformly over $[10^{-6}, 10]$.

\begin{figure}[t]
    \centering
    \includegraphics[width=0.75\linewidth]{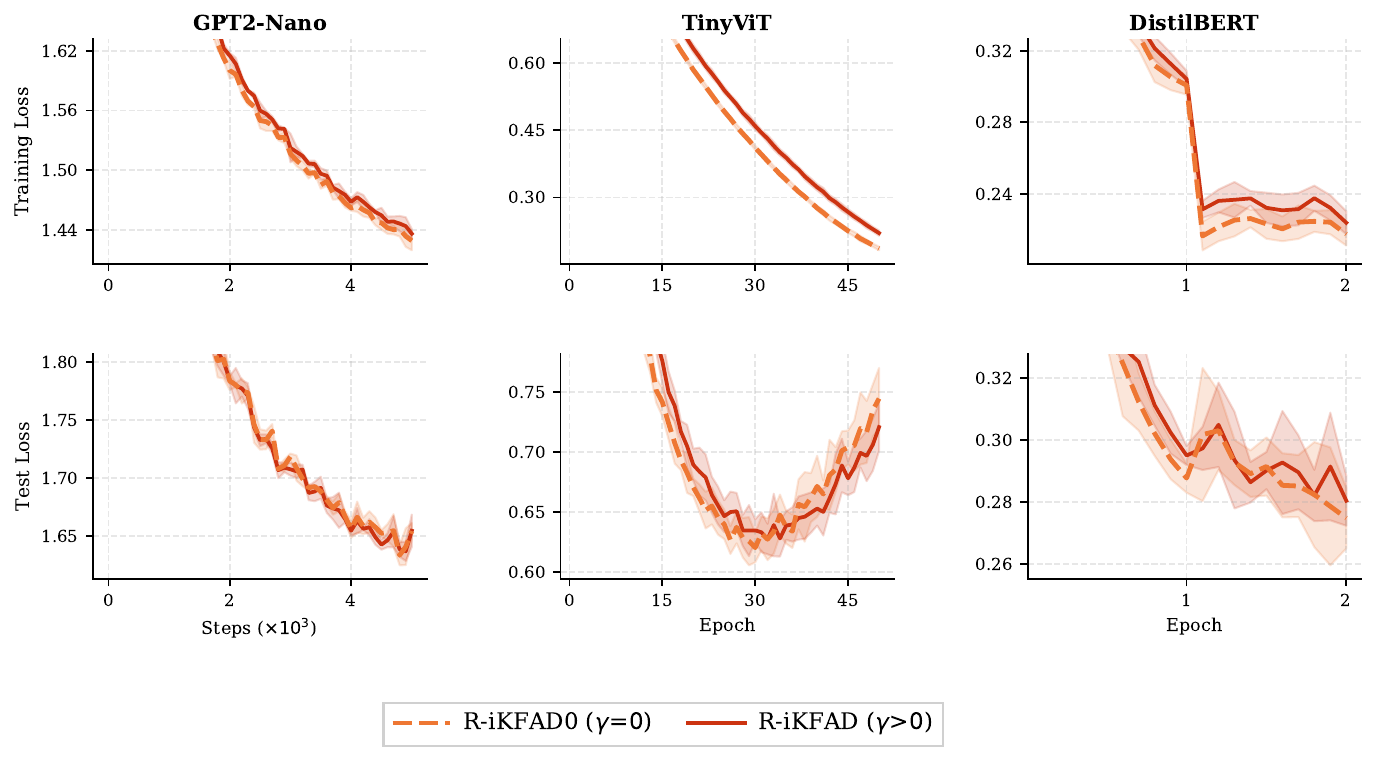}
    \caption{Training and test loss for R-iKFAD ($\gamma > 0$) vs.\ R-iKFAD0 ($\gamma = 0$) on GPT2-Nano, TinyViT, and DistilBERT (SST-2), averaged over 10 seeds (shaded = std). The two variants are indistinguishable in performance, supporting R-iKFAD0 as the recommended default.}
    \label{fig:ablation}
\end{figure}

\subsection{Robustness to Hyperparameters}

Figure~\ref{fig:gammah} shows $\gamma$-$h$ grid sweeps for R-iKFAD and iKFAD. Both methods exhibit broad regions of the tested grid with similar performance. In the R-iKFAD equations, an increase in momentum increases the factors $R$ and $C$, which increases $\tilde{\xi}$ and hence the damping. In these sweeps, the rank-1 approximation did not substantially reduce the range of $(\gamma,h)$ values giving low loss. The breadth of the plateau is a consequence of the feedback being negative: a step size large enough to inflate the momentum also inflates $R$ and $C$, and thus $\tilde{\xi}$, so the damping increases proportionally. The same mechanism explains why $\gamma$ adds little in the ablation study of Section~\ref{sec:ablation}: a constant damping term acts along a direction which the adaptive term already covers, but without responding to the change of state.

\begin{figure}[!h]
    \centering
    \includegraphics[width=0.75\linewidth]{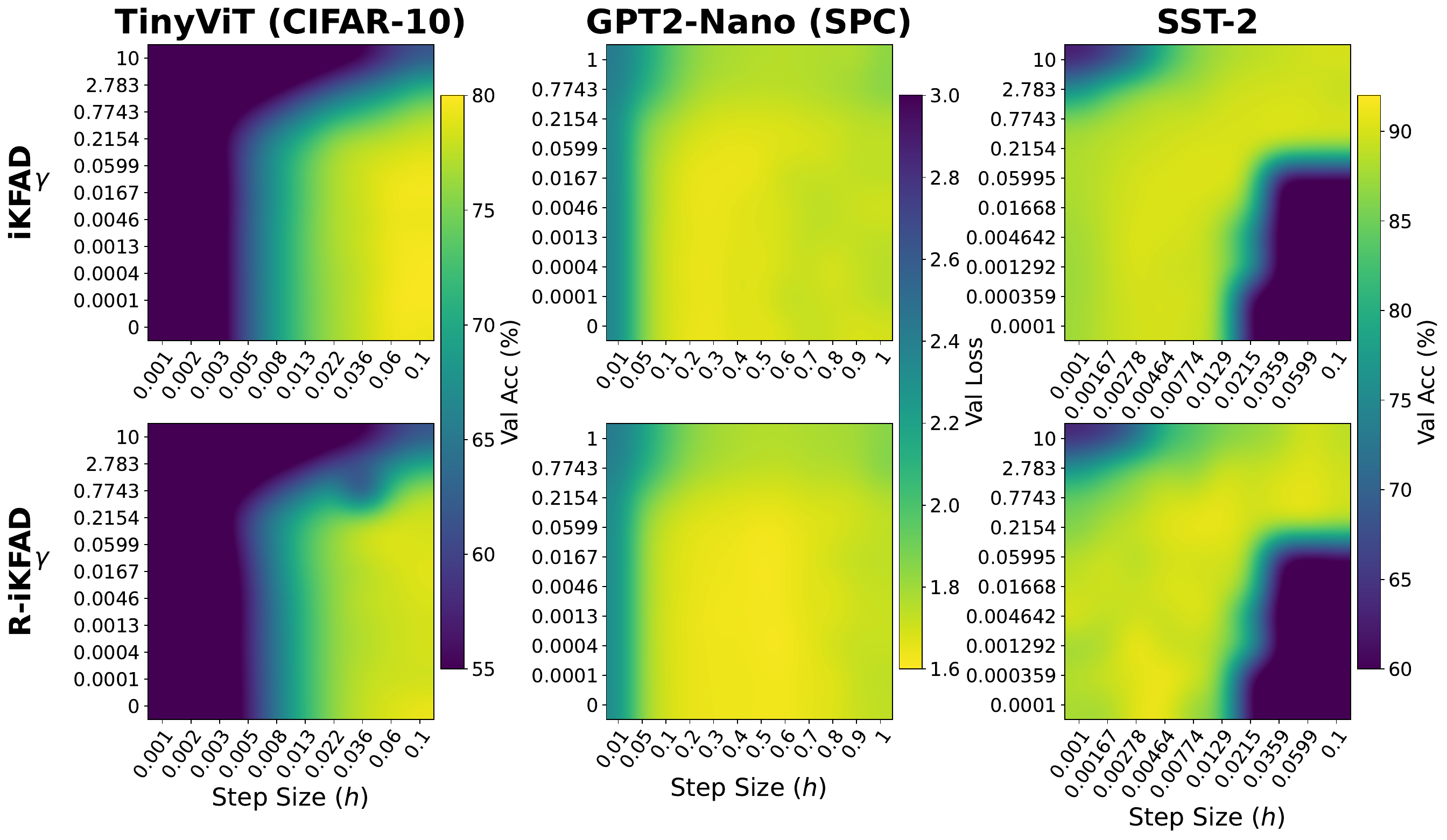}
    \caption{Two-dimensional $\gamma$--$h$ sensitivity for iKFAD (top) and R-iKFAD (bottom). Each cell shows best performance at a fixed $(h,\gamma)$ grid point. Brighter~=~better. R-iKFAD retains broad high-performance plateaus.}
    \label{fig:gammah}
\end{figure}

\FloatBarrier
\subsection{How Close is \texorpdfstring{$\tilde{\xi}$}{xi tilde} to the Best Rank-1 Friction?}
Factorization could give a poor approximation for two reasons: $\xi$ (iKFAD) may not be close to rank-1,
  or the construction of $\tilde{\xi}$ (R-iKFAD) may miss a good rank-1 approximation. By the Eckart-Young
  theorem, the smallest relative error of any rank-1 approximation is $\sqrt{1-\sigma_1^2/\sum_i\sigma_i^2}$.

  Figure~\ref{fig:eckart} compares this minimum with the error of $\tilde{\xi}$. We run iKFAD0
  and update R-iKFAD's factors $(R_\ell,C_\ell)$ alongside it, without using them in the
  optimizer updates. This lets us compare $\xi$ and $\tilde{\xi}$ at the same iterates.
  Across the four benchmarks, the error of $\tilde{\xi}$ is between $0.07$ and $0.85$
  percentage points above the minimum. On GPT2-S, the median gap across the $74$
  two-dimensional layers is $0.06$ percentage points, and $61$ layers are within one
  percentage point of the minimum. Relative to the minimum error, the gap is largest
  where $\xi$ is nearly rank-1 and the minimum is close to zero.

  The stable rank $\|\xi\|_F^2/\|\xi\|_2^2$ ranges from $1.18$ on TinyViT to $2.01$ on
  GPT2-Nano. Within GPT2-S, it varies by module: early attention value projections have
  stable ranks close to one, while query and key projections in the middle of the network
  have the largest stable ranks. These results suggest that most of the approximation
  error comes from the rank-1 restriction itself. The construction of $\tilde{\xi}$ adds
  little further error.
  
\begin{figure}[!h]
    \centering
    \includegraphics[width=\linewidth]{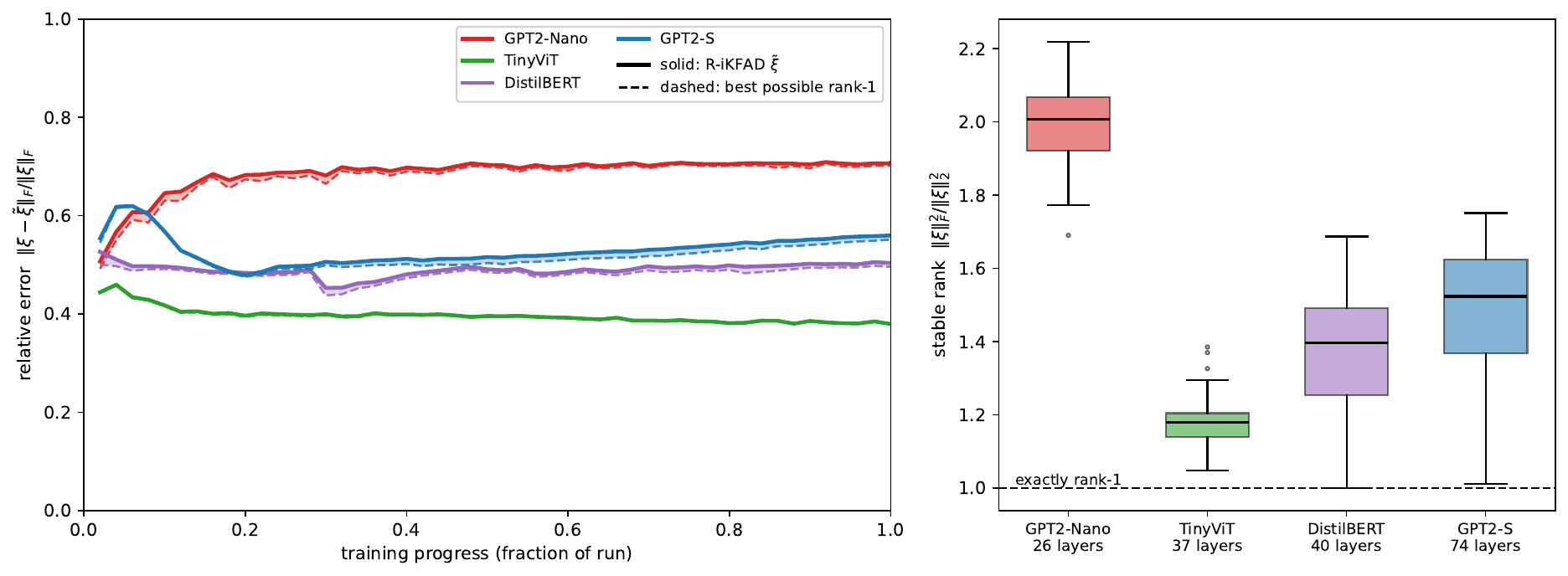}
    \caption{R-iKFAD's rank-1 friction against the best rank-1 approximation of the same tensor.
    \textbf{Left:} layer-averaged relative error $\|\xi - \tilde{\xi}\|_F/\|\xi\|_F$ (solid)
    against the Eckart--Young floor $\sqrt{1-\sigma_1^2/\sum_i\sigma_i^2}$ (dashed). The shaded region is the excess attributable to the construction, and is never wider than $0.85$ pp. The
    horizontal axis is the fraction of each run completed, as the four benchmarks used different
    step budgets. \textbf{Right:} distribution over layers of the stable rank of $\xi$ at the
    final step, where $1$ would be exactly rank-1. Measurement is taken along iKFAD0 trajectories carrying
    R-iKFAD's factors as shadow states.}
    \label{fig:eckart}
\end{figure}

\section{Conclusion}

We introduced \textbf{R-iKFAD}, which replaces iKFAD's full friction tensor with a rank-1 outer
product of momentum row and column sums, reducing per-layer friction state from
$\mathcal{O}(mn)$ to $\mathcal{O}(m+n)$ while retaining a full momentum buffer. Across
GPT2-Nano, TinyViT, DistilBERT and GPT2-S, R-iKFAD matches iKFAD in convergence speed, final
performance and hyperparameter robustness at roughly half the auxiliary optimizer memory. The
accompanying analysis, unlike the usual treatment of momentum-based dynamics, does not require
positive linear damping. Under strong convexity we prove exponential convergence when
$\gamma>0$ and convergence to the minimizer when $\gamma=0$. In the undamped case we demonstrate
matching upper and lower bounds on the energy, of order $t^{-1}$ when
$\epsilon_{\mathrm{stab}}=0$ and of order $t^{-1/2}$ when $\epsilon_{\mathrm{stab}}>0$.  Numerical integration of the dynamics agrees with both exponents. The analysis is continuous-time and deterministic, and assumes strong convexity. Extending it to the discrete-time, stochastic-gradient setting remains open.

\section*{Reproducibility Statement}

Python source code for all four benchmarks is available at
\url{https://github.com/rajit906/rank-1-ikfad}. It contains the iKFAD and R-iKFAD
implementations, training scripts, and hyperparameter sweep scripts, as well as
per-benchmark environment specifications including seed numbers. Visualization code
for Figures~\ref{fig:losses}, \ref{fig:ablation} and~\ref{fig:acc} is included in a
notebook. \texttt{theory\_numerics/} contains the MATLAB code behind the numerical
verification of the rates in Figure~\ref{fig:g0_decay}. Appendix~\ref{app:hyperparams}
documents the hyperparameter sweep and training configurations.

\section*{Acknowledgements}
Part of this research was supported by the ProbAI Hub. The authors acknowledge the use of resources provided by the Isambard-AI National AI Research Resource (AIRR) \citep{McAlWo2024}. Isambard-AI is operated by the University of Bristol and is funded by the UK Government's Department for Science, Innovation and Technology (DSIT) via UK Research and Innovation; and the Science and Technology Facilities Council [ST/AIRR/I-A-I/u6ih].

\bibliographystyle{plainnat}
\bibliography{References/references}
\appendix
\section{Update Rules for Splitting Discretisation}\label{app:splitting}

This appendix gives the explicit sub-step updates for the BACD splitting introduced in Section~\ref{sec:method}, and the treatment of the one sub-flow that does not admit a closed-form joint solution.

\begin{table}[h]
\centering
\caption{Analytical sub-step update rules for the BACD splitting sequence over time step $h$.}
\label{tab:splitting_updates}
\small
\begin{tabular}{c l}
\toprule
Sub-step & Update Equation \\
\midrule
\textbf{B} & $P^{(B)} = P_n - h\,\nabla f(X_n)$ \\[6pt]
\textbf{A} & $X_{n+1} = X_n + h\,P^{(B)}$ \\[6pt]
\textbf{C} & $\tilde{\xi}_n = \dfrac{R_n C_n^\top}{\mathbf{1}_m^\top R_n + \epsilon_{\mathrm{stab}}}$ \\[10pt]
           & $P^{(1)} = P^{(B)} \odot \exp\!\left(-\tfrac{h}{2}\,\tilde{\xi}_n\right)$ \\[10pt]
           & $R_{n+1} = \mathrm{e}^{-\alpha h} R_n + \dfrac{1-\mathrm{e}^{-\alpha h}}{\mu\alpha}[P^{(1)}]^2\mathbf{1}_n$ \\[10pt]
           & $C_{n+1} = \mathrm{e}^{-\alpha h} C_n + \dfrac{1-\mathrm{e}^{-\alpha h}}{\mu\alpha}([P^{(1)}]^2)^\top\mathbf{1}_m$ \\[10pt]
           & $\tilde{\xi}_{n+1} = \dfrac{R_{n+1} C_{n+1}^\top}{\mathbf{1}_m^\top R_{n+1} + \epsilon_{\mathrm{stab}}}$ \\[10pt]
           & $P^{(2)} = P^{(1)} \odot \exp\!\left(-\tfrac{h}{2}\,\tilde{\xi}_{n+1}\right)$ \\[6pt]
\textbf{D} & $P_{n+1} = P^{(2)} \odot \mathrm{e}^{-\gamma h}$ \quad (\textit{identity when $\gamma=0$}) \\
\bottomrule
\end{tabular}
\end{table}
\noindent Within sub-flow C, no closed-form joint solution exists for coupled $P$ and $(R, C)$. We apply a symmetric sequence: evaluating $\exp(-\tfrac{h}{2}\tilde{\xi})$ twice around an exact exponential ODE factor update improves integration accuracy within sub-flow C alone. Using the intermediate momentum $P^{(1)}$ for factor updates reduces forcing magnitude.

\section{Optimal Hyperparameters}
\label{app:hyperparams}

Across models and tasks, the best-performing scales for $\mu$ (iKFAD0, R-iKFAD0) varied by several orders of magnitude. This variation is likely related to the scale of the momentum variables during training, since $\mu$ controls the strength of the adaptive friction terms. Normalizing the momentum variables across iterations could potentially reduce this sensitivity, but we leave this for future work.

Hyperparameters were selected by minimum validation loss during the Optuna sweeps. All reported results are computed on held-out test data.

Experimental configurations:
\begin{itemize}
    \item \textbf{TinyViT -- CIFAR-10}: batch size 128, 25 epochs, 100 trials. $h \in [10^{-6}, 10^{-1}]$, $\alpha \in [10^{-3}, 1]$, $\mu \in [10^{-8}, 10]$ for iKFAD0 and R-iKFAD0.
    \item \textbf{DistilBERT -- SST-2}: batch size 16, 2 epochs, 80 trials. $h \in [10^{-6}, 10^{-1}]$, $\alpha \in [10^{-3}, 1]$, $\mu \in [10^{-8}, 10]$ for iKFAD0 and R-iKFAD0.
    \item \textbf{GPT2-Nano -- Shakespeare}: batch size 16, 5000 steps for the hyperparameter sweeps and final reported runs, 500 trials. $h \in [10^{-6}, 5 \times 10^{-1}]$, $\alpha \in [10^{-5}, 10]$, $\mu \in [10^{-8}, 10]$ for iKFAD0 and R-iKFAD0.
    \item \textbf{GPT2-S -- OpenWebText}: batch size 16, 75001 steps for the hyperparameter sweeps and 200001 steps for the final reported runs, 80 trials. $h \in [10^{-6}, 10^{-1}]$, $\alpha \in [10^{-3}, 1]$, $\mu \in [10^{-8}, 10]$ for iKFAD0 and $\mu \in [10^{-9}, 10^{-2}]$ for R-iKFAD0.
\end{itemize}

\begin{table}[H]
\centering
\footnotesize
\setlength{\tabcolsep}{4pt}
\renewcommand{\arraystretch}{1.15}
\caption{Optimized hyperparameters for all experiments. All iKFAD0 and R-iKFAD0 entries have $\gamma=0$. Dashes indicate parameters not used by the method.}
\label{tab:opt_params}
\begin{tabular}{@{}lccccc@{}}
\toprule
 & $h$ & $\alpha$ & $\mu$ & $\beta_1$ & $\beta_2$ \\
\midrule

\multicolumn{6}{c}{\textbf{TinyViT (CIFAR-10)}} \\
\midrule
iKFAD0   & 0.0725 & 0.0899 & $1.2643 \times 10^{-5}$ & -- & -- \\
R-iKFAD0 & 0.0860 & 0.5327 & $9.4765 \times 10^{-7}$ & -- & -- \\
Adam     & $5.5311 \times 10^{-4}$ & -- & -- & 0.8608 & 0.8852 \\
Adafactor-m & $5.6050 \times 10^{-4}$ & -- & -- & 0.8907 & -- \\
\midrule

\multicolumn{6}{c}{\textbf{DistilBERT (SST-2)}} \\
\midrule
iKFAD0   & 0.0165 & 0.0574 & $9.8948 \times 10^{-6}$ & -- & -- \\
R-iKFAD0 & 0.0423 & 0.0060 & $8.1874 \times 10^{-7}$ & -- & -- \\
Adam     & $3.82 \times 10^{-5}$ & -- & -- & 0.9187 & 0.9825 \\
Adafactor-m & $2.5638 \times 10^{-5}$ & -- & -- & 0.9622 & -- \\
\midrule

\multicolumn{6}{c}{\textbf{GPT2-Nano (Shakespeare)}} \\
\midrule
iKFAD0   & 0.4941 & 2.1955 & $4.5521 \times 10^{-6}$ & -- & -- \\
R-iKFAD0 & 0.4842 & 2.8475 & $1.9407 \times 10^{-6}$ & -- & -- \\
Adam     & $1.6803 \times 10^{-3}$ & -- & -- & 0.8876 & 0.9265 \\
Adafactor-m & $2.9923 \times 10^{-3}$ & -- & -- & 0.8893 & -- \\
\midrule

\multicolumn{6}{c}{\textbf{GPT2-S (OpenWebText)}} \\
\midrule
iKFAD0   & 0.0996 & 0.0476 & $1.0429 \times 10^{-5}$ & -- & -- \\
R-iKFAD0 & 0.0918 & 0.0579 & $1.3637 \times 10^{-6}$ & -- & -- \\
Adam     & $6.4723 \times 10^{-4}$ & -- & -- & 0.8945 & 0.9945 \\
Adafactor-m & $7.2777 \times 10^{-4}$ & -- & -- & 0.8865 & -- \\
\midrule

\bottomrule
\end{tabular}
\end{table}

\noindent For all experiments, Adam was swept over $h \in [10^{-6}, 10^{-2}]$, $\beta_1 \in [0.85, 0.999]$, $\beta_2 \in [0.85, 0.999]$. Adafactor-m was swept over $h \in [10^{-5}, 10^{-2}]$ and $\beta_1 \in [0.85, 0.999]$, with no relative step or warmup. Adafactor-m has no tunable $\beta_2$: as originally proposed~\citep{ShSt2018} and in standard practice, the second-moment decay follows the prescribed schedule $\beta_{2,t} = 1 - t^{-0.8}$ (\texttt{decay\_rate}$\,=-0.8$), which is why its $\beta_2$ entries are left blank. We enable $\beta_1 > 0$ rather than the conventional momentum-free setting so that Adafactor-m occupies the same auxiliary-memory regime as R-iKFAD (Table~\ref{tab:memory}). All sweeps used Optuna's~\citep{AkSaYa+2019} TPE sampler. Table~\ref{tab:opt_params} reports the best hyperparameters found. iKFAD0 and R-iKFAD0 are the $\gamma=0$ variants. All such entries satisfy $\gamma=0$ by definition. For iKFAD0 and R-iKFAD0, $h$ denotes the step size and $\mu$ the adaptive damping coefficient. For Adam and Adafactor-m, $h$ is the learning rate.

\section{Memory Analysis}
\label{app:memory}
\begin{table}[H]
\centering
\caption{Optimizer auxiliary state (MiB, float32), computed over all model parameters. R-iKFAD achieves approximately $2\times$ savings over both iKFAD and Adam because the rank-1 friction factors $(R_\ell, C_\ell)$ are negligible relative to the full momentum buffer. The Adafactor-m baseline used throughout enables its optional first moment ($\beta_1>0$, see Section~\ref{sec:experiments}) and therefore carries an auxiliary footprint identical to that of R-iKFAD, making the two directly comparable at matched memory. Momentum-free Adafactor is included for reference.}
\label{tab:memory}
\small
\begin{tabular}{lrrrrrr}
\toprule
Model & $N$ & Adam & Adafactor-m & Adafactor & iKFAD & R-iKFAD \\
\midrule
TinyViT      &  4.0M  &  30.7  &  15.6  &   0.2  &  30.7  &  15.6 ($2.0\times$) \\
DistilBERT   & 67.0M  & 510.8  & 256.1  &   0.7  & 510.8  & 256.1 ($2.0\times$) \\
GPT2-Nano    &  0.8M  &   6.1  &   3.1  & $<0.1$  &   6.1  &   3.1 ($2.0\times$) \\
GPT2-S       & 124.4M & 948.9  & 475.4  &   0.9  & 948.9  & 475.4 ($2.0\times$) \\
\bottomrule
\end{tabular}
\end{table}

\section{Accuracy Curves}
\label{app:accuracy}

Figure~\ref{fig:acc} gives the accuracy curves corresponding to the loss curves of
Figure~\ref{fig:losses}, for the two classification benchmarks. They are reported here rather
than in Section~\ref{sec:experiments} because they follow the loss curves closely.

\begin{figure}[H]
    \centering
    \includegraphics[width=0.55\linewidth]{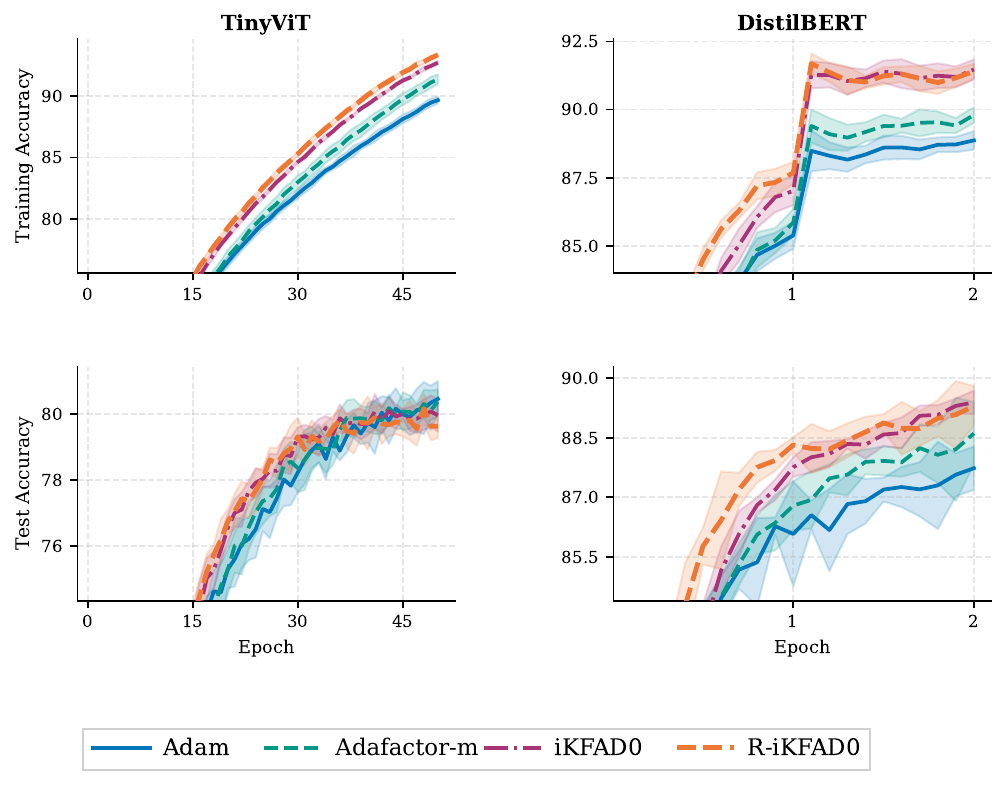}
    \caption{Accuracy averaged over 10~seeds. R-iKFAD0 matches iKFAD0 and Adam. Error bars denote one std.}
    \label{fig:acc}
\end{figure}

\section{Prefactor Dependence}
\label{app:g0_prefactors}

The prefactors depend on $\alpha$, $\mu$ and $\epsilon_{\mathrm{stab}}$ as the calculation of
Section~\ref{app:g0_formal_rate} predicts. If the asymptotics are $\mathcal{E}\sim c_1\alpha\mu/t$ and
$\mathcal{E}\sim c_2\alpha\mu\sqrt{\epsilon_{\mathrm{stab}}/2t}$, then the ratios in
Table~\ref{tab:g0_prefactors} are approximately constant. They were evaluated at $t=10^4$ with $m=3$, $n=2$. The
crossover row indicates the last downward crossing of $\mathbf1^\top R=\epsilon_{\mathrm{stab}}$.

\begin{table}[H]
\centering
\caption{Parameter dependence of the prefactors, evaluated at $t=10^4$ with $m=3$, $n=2$. In the first row
$\alpha$ and $\mu$ each range over a factor of $4$, so $\alpha\mu$ ranges over a factor of $16$, yet
the ratio moves by less than $30\%$.}
\label{tab:g0_prefactors}
\small
\begin{tabular}{lccc}
\toprule
quantity & parameters varied & mean & spread (max/min) \\
\midrule
$t\mathcal{E}/(\alpha\mu)$, \ $\epsilon_{\mathrm{stab}}=0$
  & $\alpha,\mu\in\{\tfrac12,1,2\}$ & $4.69$ & $1.29$ \\
$\mathcal{E}\sqrt t\big/\bigl(\alpha\mu\sqrt{\epsilon_{\mathrm{stab}}/2}\bigr)$, \ $\epsilon_{\mathrm{stab}}>0$
  & $\alpha,\mu\in\{\tfrac12,1\}$, $\epsilon_{\mathrm{stab}}\in\{0.05,0.2\}$ & $2.31$ & $1.14$ \\
$\mathcal{E}\big/(\alpha\mu\,\epsilon_{\mathrm{stab}})$ at the crossing
  & $\alpha,\mu\in\{\tfrac12,1\}$, $\epsilon_{\mathrm{stab}}\in\{0.05,0.2\}$ & $0.87$ & $1.21$ \\
\bottomrule
\end{tabular}
\end{table}

The third row places the transition between the two regimes at
$\mathcal{E}\approx0.87\,\alpha\mu\,\epsilon_{\mathrm{stab}}$, confirming the crossover scale. For the
scalar system $m=n=1$ with $\alpha=\mu=1$, evaluating at $t=10^5$ gives $t\mathcal{E}/(\alpha\mu)=0.909$
and $\mathcal{E}\sqrt t/(\alpha\mu\sqrt{\epsilon_{\mathrm{stab}}/2})=0.884$, both within about $10\%$ of
the value $1$ the averaging calculation gives. For $m,n>1$ the same relations hold with a
dimension-dependent coefficient.

\section{Numerical Setup for Section~\ref{app:g0_numerics}}
\label{app:g0_setup}

The initial entries of $X$ and $P$ were independent normal variables with standard deviations
$0.4$ and $0.3$, and $R(0)=C(0)=0$ throughout.
The parameter sweeps used MATLAB R2026a, \texttt{rng(0)}, and \texttt{ode89} with relative
tolerance $10^{-10}$ and absolute tolerance $10^{-12}$. When
$\epsilon_{\mathrm{stab}}=0$ and $\mathbf1^\top R=0$ the code set $\widetilde\xi=0$, and used
\eqref{eq:xi_def} once $\mathbf1^\top R>0$.
The curves in Figure~\ref{fig:g0_decay} were computed by BACD splitting with $m=4$, $n=3$,
$\alpha=\mu=1$, $h=0.02$, final time $10^{7}$ and random seed $1$. The additional runs behind
Table~\ref{tab:g0_slopes} used dimensions $2\times2$ and $6\times5$ with the same $A$, $X(0)$
and $P(0)$. 

As consistency checks, for $m=3$, $n=2$ and $\alpha=\mu=1$, the largest absolute difference
between the \texttt{ode89} and BACD slopes was $0.018$ over the three fitted intervals.
Repeating the six BACD fits at $h=0.01$ instead of $h=0.02$ changed each slope by less than
$8.9\times10^{-4}$.

\section{Rank-1 iKFAD Convergence Analysis}\label{app:proofs}

We analyze the continuous-time Rank-1 iKFAD dynamics for a matrix parameter $X \in \mathbb{R}^{m \times n}$ (the vector case follows by flattening). The system is given by
\begin{subequations}
\begin{align}
\dot{X} &= P, \label{eq:rk1_x}\\
\dot{P} &= -\nabla f(X) - \tilde{\xi} \odot P - \gamma P,\\
\dot{R} &= \frac{1}{\mu} [P]^2 \mathbf{1}_n - \alpha R, \label{eq:rk1_R}\\
\dot{C} &= \frac{1}{\mu} ([P]^2)^\top \mathbf{1}_m - \alpha C, \label{eq:rk1_C}
\end{align}
\end{subequations}
where $P \in \mathbb{R}^{m \times n}$, $R \in \mathbb{R}^m$, $C \in \mathbb{R}^n$, and $\gamma \ge 0$, $\alpha, \mu > 0$ are hyperparameters. Boundedness (Lemma~\ref{lem:rk1_bounded}) holds for $\gamma \ge 0$. The exponential convergence results below additionally require $\gamma > 0$. The rank-1 friction tensor $\tilde{\xi} \in \mathbb{R}^{m \times n}$ is defined element-wise as
\[
\tilde{\xi}_{ij} = \frac{R_i C_j}{\sum_{k=1}^m R_k + \epsilon_{\mathrm{stab}}},
\]
with $\epsilon_{\mathrm{stab}} > 0$ a small constant preventing division by zero. The notation $\odot$ denotes element-wise multiplication, $[\,\cdot\,]^2$ element-wise squaring, and $\langle \cdot, \cdot \rangle_F$ the Frobenius inner product.

\subsection{Dynamical properties and boundedness}

Throughout this appendix $X^*$ denotes a global minimizer of $f$. Under the hypotheses used below such a point exists, because $f$ is continuous and coercive, and it satisfies
\begin{equation}
\nabla f(X^*) = 0 \qquad\text{and}\qquad f(X) \ge f(X^*) \quad \text{for all } X. \label{eq:xstar}
\end{equation}

\begin{lemma}[Boundedness of solutions]\label{lem:rk1_bounded}
Assume $f \in C^2$ (so $\nabla f$ is automatically locally Lipschitz), $f(X) \to +\infty$ as $\|X\|_F \to +\infty$, and let $X^*$ be a global minimizer of $f$ as in \eqref{eq:xstar}. Then for any initial condition $(X_0, P_0, R_0, C_0)$ with $R_0, C_0 \ge 0$ entrywise, the solution of \eqref{eq:rk1_x}--\eqref{eq:rk1_C} is well-defined for all $t \ge 0$, and there exists a compact set $\mathcal{K}$ containing the trajectory. In particular, with
\[
P_{\max}^2 := 2\bigl(f(X_0) - f(X^*) + \tfrac12\|P_0\|_F^2\bigr), \quad
R_{\max} := \|R_0\| + \frac{P_{\max}^2}{\alpha\mu}, \quad
C_{\max} := \|C_0\| + \frac{P_{\max}^2}{\alpha\mu},
\]
all of which are finite and independent of $t$, we have for all $t \ge 0$
\[
0 \le R_i(t) \le R_{\max},\quad 0 \le C_j(t) \le C_{\max},\quad \|P(t)\|_F \le P_{\max},
\]
and moreover the friction tensor obeys the bound
\begin{equation}
0 \le \tilde{\xi}_{ij}(t) \le \xi_{\max} := \min\Bigl\{ C_{\max},\ \frac{R_{\max}C_{\max}}{\epsilon_{\mathrm{stab}}} \Bigr\}. \label{eq:xi_cone_cts}
\end{equation}
Note that $P_{\max}$ depends on neither $\mu$ nor $\gamma$.
\end{lemma}

\begin{proof}
The ODEs \eqref{eq:rk1_R}--\eqref{eq:rk1_C} have non-negative forcing and linear decay, so $R_0, C_0 \ge 0$ implies $R_i(t), C_j(t) \ge 0$ for all $t \ge 0$, and hence $\tilde{\xi}_{ij} \ge 0$ throughout.

Consider the reduced energy
\[
\mathcal{E}(X,P) = f(X) - f(X^*) + \tfrac12\|P\|_F^2,
\]
which omits the $R,C$ contributions. Differentiating along trajectories and cancelling the gradient terms,
\begin{equation}
\dot{\mathcal{E}} = \langle \nabla f(X), P\rangle_F + \langle P, -\nabla f(X) - \tilde{\xi}\odot P - \gamma P\rangle_F
= -\gamma\|P\|_F^2 - \sum_{i,j}\tilde{\xi}_{ij}P_{ij}^2 \le 0, \label{eq:E_decay}
\end{equation}
both terms being non-positive because $\gamma \ge 0$ and $\tilde{\xi} \ge 0$ entrywise. No bound on $\|P\|_F$ is used in deriving \eqref{eq:E_decay}, so the estimate is not circular. Consequently $\mathcal{E}(t) \le \mathcal{E}(0)$ for all $t$, which yields at once
\[
\|P(t)\|_F^2 \le 2\mathcal{E}(0) =: P_{\max}^2,
\]
and, since $f(X) - f(X^*) \le \mathcal{E}(0)$ with $f$ coercive, confinement of $X(t)$ to a compact sublevel set of $f$.

With $\|P\|_F \le P_{\max}$ now established independently, the factors follow from the explicit linear filter. Integrating \eqref{eq:rk1_R},
\[
R(t) = \mathrm{e}^{-\alpha t}R_0 + \frac{1}{\mu}\int_0^t \mathrm{e}^{-\alpha(t-s)}[P(s)]^2\mathbf{1}_n\,\diff s,
\]
and using the sharp, dimension-free estimate $\bigl\|[P]^2\mathbf{1}_n\bigr\| \le \|P\|_F^2 \le P_{\max}^2$ (the entries of $[P]^2\mathbf{1}_n$ are non-negative and sum to $\|P\|_F^2$, so its $2$-norm is at most its $1$-norm),
\[
\|R(t)\| \le \|R_0\| + \frac{P_{\max}^2}{\alpha\mu} =: R_{\max},
\]
with the analogous bound $C_{\max} := \|C_0\| + P_{\max}^2/(\alpha\mu)$. All four variables therefore remain in a compact set $\mathcal{K}$ for all $t \ge 0$, and global existence follows from the standard continuation criterion for locally Lipschitz vector fields.

Finally we establish \eqref{eq:xi_cone_cts}. Two bounds are possible. Bounding the denominator below by $\epsilon_{\mathrm{stab}}$ gives $\tilde{\xi}_{ij} \le R_{\max}C_{\max}/\epsilon_{\mathrm{stab}}$, which diverges as $\epsilon_{\mathrm{stab}} \downarrow 0$. The second is sharper and uniform in $\epsilon_{\mathrm{stab}}$: since every $R_k \ge 0$ we have $R_i \le \mathbf{1}_m^\top R$, hence
\[
\tilde{\xi}_{ij} = \frac{R_i}{\mathbf{1}_m^\top R + \epsilon_{\mathrm{stab}}}\, C_j \le C_j \le C_{\max},
\]
and \eqref{eq:xi_cone_cts} is the minimum of the two. The constant $\epsilon_{\mathrm{stab}}>0$ keeps the friction well defined at $R=0$, which is how the implementation initializes the factors. The case $\epsilon_{\mathrm{stab}}=0$ is treated separately in Lemma~\ref{lem:balanced}. The bound \eqref{eq:xi_cone_cts} is independent of $\epsilon_{\mathrm{stab}}$, so the stabilizer does not affect the exponential rate of Theorem~\ref{thm:rk1_exp}. At $\gamma=0$, however, it does determine the algebraic exponent (Theorem~\ref{thm:rates}).
\end{proof}

\subsection{Geometric convergence}

\noindent\textbf{Theorem~\ref{thm:rk1_exp}.} Suppose that $f$ is $m_f$-strongly convex and has an
$L_f$-Lipschitz gradient. Let
$\gamma,\alpha,\mu,\epsilon_{\mathrm{stab}}>0$, and let $R(0)$ and
$C(0)$ be componentwise nonnegative. Then there are constants
$M,\rho>0$ such that
\begin{equation}
 f(X(t))-f(X^*)+\|X(t)-X^*\|_F^2+\|P(t)\|_F^2
 +\|R(t)\|_2^2+\|C(t)\|_2^2
 \leq M e^{-\rho t}
\end{equation}
for every $t\geq0$.

\begin{proof}[Proof of Theorem~\ref{thm:rk1_exp}]
Put $Y=X-X^*$ and $d_{\max}=\gamma+\xi_{\max}$. Choose
\begin{equation}
0<\delta\leq
\min\left\{\frac{\sqrt{m_f}}{2},
\frac{\gamma}{2\left(1+d_{\max}^2/(2m_f)\right)}
\right\}.
\label{eq:positive_gamma_delta}
\end{equation}
Here $\xi_{\max}$ is the bound in Lemma~\ref{lem:rk1_bounded}.
Define
\[
V=f(X)-f(X^*)+\frac12\|P\|_F^2
  +\delta\langle Y,P\rangle_F .
\]
Young's inequality gives
\begin{align}
V
&\geq
\frac{m_f}{4}\|Y\|_F^2+\frac14\|P\|_F^2.
\label{eq:positive_gamma_V_lower}
\end{align}
It also gives
\begin{equation}
V\leq \left(\frac{L_f}{2}+\frac{m_f}{4}\right)\|Y\|_F^2
       +\frac34\|P\|_F^2.
\label{eq:positive_gamma_V_upper}
\end{equation}

Differentiating along the equations gives
\begin{align}
\dot V
&\leq -(\gamma-\delta)\|P\|_F^2
-\delta m_f\|Y\|_F^2+\delta d_{\max}\|Y\|_F\|P\|_F
\notag\\
&\leq
-\frac{\delta m_f}{2}\|Y\|_F^2-\frac{\gamma}{2}\|P\|_F^2.
\label{eq:positive_gamma_V_decay}
\end{align}
The last line follows from Young's inequality and
\eqref{eq:positive_gamma_delta}.

Set
\[
\rho_0:=\min\left\{
\frac{2\delta m_f}{2L_f+m_f},\frac{2\gamma}{3}
\right\}.
\]
Equations~\eqref{eq:positive_gamma_V_upper} and
\eqref{eq:positive_gamma_V_decay} imply
\[
\dot V\leq-\rho_0V.
\]
Therefore $V(t)\leq e^{-\rho_0t}V(0)$.

The lower bound~\eqref{eq:positive_gamma_V_lower} now gives
exponential decay of $X-X^*$ and $P$. The variation-of-constants
formulas for $R$ and $C$ then give their exponential decay. For
example,
\[
R_i(t)=e^{-\alpha t}R_i(0)
+\frac1\mu\int_0^t e^{-\alpha(t-s)}
       \sum_{k=1}^nP_{ik}(s)^2\,\diff s .
\]
The convolution on the right decays exponentially. The same argument
applies to $C$. Taking, for example,
\[
0<\rho<\min\{\rho_0,2\alpha\}
\]
and increasing $M$ if necessary proves
\eqref{eq:positive_gamma_result}. The same bounds prevent finite-time
escape, so the solution is defined for every $t\geq0$.
\end{proof}

The proof uses the fixed term $-\gamma\|P\|_F^2$ to control the
position--momentum cross term. When $\gamma=0$, this control is absent.
The next subsection therefore treats this case separately.

\subsection{The case \texorpdfstring{$\gamma=0$}{gamma=0}}
In our experiments we observed that $\gamma=0$ gave optimal or near-optimal results.
The next theorem proves convergence of the continuous system in this case.

\noindent\textbf{Theorem~\ref{thm:g0_conv}.} Suppose that $f$ is strongly convex and has a locally Lipschitz gradient. Let $\alpha,\mu,\epsilon_{\mathrm{stab}}>0$, and suppose that $R(0)$ and $C(0)$ are componentwise nonnegative. Then every solution of \eqref{eq:rk1_x}--\eqref{eq:rk1_C} with $\gamma=0$ converges to
\[
(X,P,R,C)=(X_\star,0,0,0),
\]
where $X_\star$ is the unique minimizer of $f$.

\begin{proof}[Proof of Theorem~\ref{thm:g0_conv}]
Set
\[
\mathcal H(X,P)=f(X)+\frac12\|P\|_F^2.
\]
Along a solution with $\gamma=0$,
\begin{align}
\dot{\mathcal H}
&=\langle \nabla f(X),P\rangle_F
 +\langle P,-\nabla f(X)-\tilde\xi\mathbin\odot P\rangle_F \notag\\
&=-\sum_{i,j}\tilde\xi_{ij}P_{ij}^2\leq0. \label{eq:g0_energy}
\end{align}
Thus $\mathcal H$ is nonincreasing. Its sublevel sets are bounded by
strong convexity. Lemma~\ref{lem:rk1_bounded} gives global existence.

The factors satisfy
\begin{align}
R_i(t)
&=e^{-\alpha t}R_i(0)
 +\frac1\mu\int_0^t e^{-\alpha(t-s)}
       \sum_{k=1}^nP_{ik}(s)^2\,\diff s, \label{eq:g0_R}\\
C_j(t)
&=e^{-\alpha t}C_j(0)
 +\frac1\mu\int_0^t e^{-\alpha(t-s)}
       \sum_{\ell=1}^mP_{\ell j}(s)^2\,\diff s. \label{eq:g0_C}
\end{align}
These formulas show that the factors remain nonnegative. Since $P$
is bounded, they also show that $R$ and $C$ are bounded. The complete
trajectory is therefore precompact.

It remains to identify the largest invariant subset of
$\{\dot{\mathcal H}=0\}$. Equation~\eqref{eq:g0_energy} and
$\tilde\xi_{ij}\geq0$ imply that
\[
\tilde\xi_{ij}P_{ij}^2=0
\quad\hbox{for every }i,j
\]
on this set. Suppose that $P_{ij}\neq0$ at some point of a trajectory
in the set. Continuity makes $P_{ij}^2>0$ on a time interval. Because
$\epsilon_{\mathrm{stab}}>0$, the displayed equality then implies
$R_iC_j=0$ throughout that interval.

This is impossible. If $R_i=0$ at any time in the interval, then
\[
\dot R_i=\frac1\mu\sum_kP_{ik}^2>0.
\]
Similarly, if $C_j=0$, then
\[
\dot C_j=\frac1\mu\sum_\ell P_{\ell j}^2>0.
\]
Moreover, the variation-of-constants formulas show that a nonnegative
factor, once positive, stays positive. Hence both $R_i$ and $C_j$
become positive in the interval, contradicting $R_iC_j=0$. We conclude
that every trajectory in $\{\dot{\mathcal H}=0\}$ has $P=0$.

If $P=0$ along an invariant trajectory, then $\nabla f(X)=0$ and $X=X^*$.
The largest invariant subset of $\{\dot{\mathcal H}=0\}$ is therefore
\[
\{(X^*,0,R,C):R,C\geq0\}.
\]
LaSalle's invariance principle gives
\[
X(t)\longrightarrow X^*,\qquad
P(t)\longrightarrow0.
\]

Finally, equations~\eqref{eq:g0_R}--\eqref{eq:g0_C} and
$P(t)\to0$ give $R(t),C(t)\to0$.
\end{proof}

\subsubsection{Algebraic rates}\label{app:g0_rates}

The formal calculation above identifies the exponents. The next result establishes them rigorously, with matching upper and lower bounds in both regimes.

Throughout this subsection $\gamma=0$ in \eqref{eq:rk1_x}--\eqref{eq:rk1_C}, and the friction is
\begin{equation}
 \tilde{\xi}_{ij}=\frac{R_iC_j}{s+\epsilon_{\mathrm{stab}}},
 \qquad s=\sum_iR_i.                                         \label{eq:friction}
\end{equation}
When $\epsilon_{\mathrm{stab}}=0$, set $\tilde{\xi}=0$ at $s=0$. Since
$\dot R_i\ge-\alpha R_i$ and $\dot C_j\ge-\alpha C_j$, nonnegative initial
factors remain nonnegative, so $s\ge0$ along the flow. We write
$\lVert\tilde{\xi}\rVert_\infty=\max_{i,j}|\tilde{\xi}_{ij}|$.

Let $X^*$ be the minimizer of $f$, and define
\begin{equation}
 \mathcal{E}(t)=f(X(t))-f(X^*)+\frac12\lVert P(t)\rVert_F^2.           \label{eq:energy}
\end{equation}

\noindent\textbf{Theorem~\ref{thm:rates}.}
Suppose that $f$ is strongly convex, its gradient is locally Lipschitz,
and it is $C^2$ in a neighbourhood of $X^*$. Let
$\gamma=0$ and $\alpha,\mu>0$, and let the initial factors be componentwise nonnegative.

If $\epsilon_{\mathrm{stab}}>0$ and $\mathcal{E}(0)>0$, there are $a,A>0$
such that
\[
 \frac a{\sqrt{1+t}}\le \mathcal{E}(t)\le\frac A{\sqrt{1+t}},
 \qquad t\ge0.
\]
If $\epsilon_{\mathrm{stab}}=0$, assume in addition that
$R(0)=C(0)=0$, and use the convention following \eqref{eq:friction}.
If $\mathcal{E}(0)>0$, there are $b,B>0$ such that
\[
 \frac b{1+t}\le \mathcal{E}(t)\le\frac B{1+t},
 \qquad t\ge0.
\]

The proof uses two estimates. A lower bound on the energy lost over a fixed
time window gives the upper bound on $\mathcal{E}$. An upper bound on the
instantaneous loss gives the lower bound on $\mathcal{E}$.

\begin{lemma}[The zero-stabilizer phase space]\label{lem:balanced}
Suppose that $\epsilon_{\mathrm{stab}}=0$ and
$\sum_iR_i=\sum_jC_j=:s$. The convention
$\tilde{\xi}=0$ at $s=0$ defines a locally Lipschitz friction tensor on
\[
 \mathcal B=\Bigl\{R\ge0,\ C\ge0,\ \sum_iR_i=\sum_jC_j\Bigr\}.
\]
Moreover,
\begin{equation}
 0\le\tilde{\xi}_{ij}\le C_j\le s.                         \label{eq:xi_s_bound}
\end{equation}
\end{lemma}

\begin{proof}
The bound follows from $R_i\le s$. For local Lipschitz continuity,
consider two points of $\mathcal B$ with totals $0\le s\le s'$. If
$s=s'=0$, both friction tensors vanish. If $s=0<s'$, then
\[
 |\tilde{\xi}_{ij}-\tilde{\xi}'_{ij}|
 =\frac{R'_iC'_j}{s'}\le s'=|s-s'|.
\]
If $s>0$, write
\[
 R=s\rho,\quad C=s\sigma,\qquad R'=s'\rho',\quad C'=s'\sigma',
\]
where the four normalized vectors lie in the appropriate simplices. Then
\[
 \left|s\rho_i\sigma_j-s'\rho'_i\sigma'_j\right|
 \le |s-s'|+s|\rho_i-\rho'_i|+s|\sigma_j-\sigma'_j|,
\]
and
\[
 s|\rho_i-\rho'_i|\le |R_i-R'_i|+|s-s'|,
 \qquad
 s|\sigma_j-\sigma'_j|\le |C_j-C'_j|+|s-s'|.
\]
Since $|s-s'|\le\lVert R-R'\rVert_1$, these estimates give the required
Lipschitz bound.
\end{proof}

\begin{lemma}[Kinetic energy on a fixed window]\label{lem:observability}
Assume the hypotheses of Theorem~\ref{thm:rates}, and suppose a solution of
\eqref{eq:rk1_x}--\eqref{eq:rk1_C} satisfies
\[
 X(t)\longrightarrow X^*,\qquad P(t)\longrightarrow0,
 \qquad \lVert\tilde{\xi}(t)\rVert_\infty\longrightarrow0.
\]
Then for every $T>0$ there are $c_0>0$ and $t_0\ge0$, both depending on
the solution, such that
\begin{equation}
 \int_t^{t+T}\lVert P(u)\rVert_F^2\,du\ge c_0\mathcal{E}(t)
 \qquad(t\ge t_0).                                           \label{eq:observability}
\end{equation}
\end{lemma}

\begin{proof}
Fix $T>0$. If the result is false, for every $k$ there is $t_k\ge k$
violating \eqref{eq:observability} with $c_0=1/k$. The integral is
nonnegative, so
$e_k=\mathcal{E}(t_k)>0$, and
\begin{equation}
 \frac1{e_k}\int_{t_k}^{t_k+T}\lVert P(u)\rVert_F^2\,du
 \longrightarrow0.                                          \label{eq:obs_contra}
\end{equation}
Set, for $0\le u\le T$,
\[
 y_k(u)=\frac{X(t_k+u)-X^*}{\sqrt{e_k}},
 \qquad
 p_k(u)=\frac{P(t_k+u)}{\sqrt{e_k}}.
\]
The assumed convergence gives $e_k\to0$.
Energy monotonicity and strong convexity make $y_k$ and $p_k$ uniformly
bounded. On bounded sets of $y$, the $C^2$ assumption gives, uniformly,
\[
 \frac{\nabla f(X^*+\sqrt{e_k}y)}{\sqrt{e_k}}
 \longrightarrow Hy,\qquad H=\nabla^2f(X^*).
\]
The equations and the uniform convergence of
$\tilde{\xi}(t_k+u)$ to zero on $[0,T]$ now give uniform bounds on
$\dot y_k$ and $\dot p_k$. Hence there is a constant $L$, independent
of $k$, such that
\[
 \lVert y_k(u)-y_k(v)\rVert_F
 +\lVert p_k(u)-p_k(v)\rVert_F\le L|u-v|
 \qquad(0\le u,v\le T).
\]
The Arzel\`a--Ascoli theorem therefore gives a subsequence on which both
sequences converge uniformly, say to $y$ and $p$. For this subsequence,
the equations are
\begin{align*}
 y_k(u)&=y_k(0)+\int_0^u p_k(v)\,dv,\\
 p_k(u)&=p_k(0)-\int_0^u
 \left[
 \frac{\nabla f(X^*+\sqrt{e_k}y_k(v))}{\sqrt{e_k}}
 +\tilde{\xi}(t_k+v)\mathbin\odot p_k(v)
 \right]dv.
\end{align*}
Uniform convergence gives $p_k\to p$. The Taylor estimate above gives
uniform convergence of the scaled gradient to $Hy$, while
$\tilde{\xi}(t_k+v)\mathbin\odot p_k(v)\to0$ uniformly. Taking the limit
in these identities gives
\[
 y(u)=y(0)+\int_0^u p(v)\,dv,
 \qquad
 p(u)=p(0)-\int_0^u Hy(v)\,dv.
\]
Thus the limits solve
\begin{equation}
 \dot y=p,\qquad \dot p=-Hy.                                 \label{eq:limit}
\end{equation}
Taylor expansion of $f$ at $X^*$ gives
\[
 \frac12\langle Hy(0),y(0)\rangle
 +\frac12\lVert p(0)\rVert_F^2=1.
\]
Equation \eqref{eq:obs_contra} gives
$\int_0^T\lVert p(u)\rVert_F^2\,du=0$. Thus $p=0$ on $[0,T]$.
Equation \eqref{eq:limit} gives $Hy=0$. Strong convexity makes $H$
positive definite, so $y=0$, contradicting the normalized energy.
\end{proof}

\begin{proof}[Proof of Theorem~\ref{thm:rates}]
\smallskip\emph{Boundedness and convergence.}
Differentiating \eqref{eq:energy} gives
\begin{equation}
 \dot{\mathcal{E}}(t)=-\sum_{i,j}\tilde{\xi}_{ij}(t)P_{ij}(t)^2.       \label{eq:edot}
\end{equation}
Thus $\mathcal{E}$ is nonincreasing. Strong convexity bounds $X$ and $P$.
Variation of constants gives
\begin{align}
 R_i(t)&=e^{-\alpha t}R_i(0)
 +\frac1\mu\int_0^te^{-\alpha(t-u)}\sum_jP_{ij}(u)^2\,du,
                                                                    \label{eq:Rfilter}\\
 C_j(t)&=e^{-\alpha t}C_j(0)
 +\frac1\mu\int_0^te^{-\alpha(t-u)}\sum_iP_{ij}(u)^2\,du.    \label{eq:Cfilter}
\end{align}
Hence the factors are also bounded. Local Lipschitz continuity, together
with boundedness, gives global existence. In the zero-stabilizer case,
local Lipschitz continuity is understood on the phase space in
Lemma~\ref{lem:balanced}.
If $\mathcal{E}(0)=0$, then $X=X^*$, $P=0$, and the energy remains zero.
Assume from now on that $\mathcal{E}(0)>0$.

Consider a complete bounded trajectory in the
$\omega$-limit set on which $\dot{\mathcal{E}}=0$. Boundedness of the complete
trajectory gives
\[
 R_i(t)=\frac1\mu\int_{-\infty}^t
 e^{-\alpha(t-u)}\sum_jP_{ij}(u)^2\,du,
 \qquad
 C_j(t)=\frac1\mu\int_{-\infty}^t
 e^{-\alpha(t-u)}\sum_iP_{ij}(u)^2\,du.
\]
If some $P_{ij}\ne0$, it remains nonzero on a short interval, and these
formulas make both $R_i$ and $C_j$ positive on a nonempty subinterval.
This makes
$\tilde{\xi}_{ij}P_{ij}^2>0$, a contradiction. Hence $P=0$ on every
such trajectory. Invariance then gives $\nabla f(X)=0$, and thus
$X=X^*$. LaSalle's invariance principle and the filter formulas give
\begin{equation}
 X(t)\to X^*,\qquad P(t)\to0,\qquad R(t),C(t)\to0.            \label{eq:convergence}
\end{equation}
For $\epsilon_{\mathrm{stab}}=0$, the two factor totals satisfy the same
scalar equation. The initialization therefore keeps the solution in
$\mathcal B$, and Lemma~\ref{lem:balanced} applies. In both cases,
\eqref{eq:convergence} gives
$\lVert\tilde{\xi}(t)\rVert_\infty\to0$.
Since the friction is also bounded,
$\lVert\tilde{\xi}\rVert_\infty\le M_0$ for a fixed $M_0$, so
\eqref{eq:edot} gives $-\dot{\mathcal{E}}\le M_0K\le2M_0\mathcal{E}$, where
$K=\lVert P\rVert_F^2$. Hence $\mathcal{E}(t)\ge\mathcal{E}(0)e^{-2M_0t}>0$.

\smallskip\emph{The factor totals.}
Let $q$ denote either $s=\sum_iR_i$ or
$\bar s=\sum_jC_j$. Both satisfy
\[
 \dot q=\frac K\mu-\alpha q.
\]
Set $w=q/\mathcal{E}$. Since $K\le2\mathcal{E}$, equation \eqref{eq:edot} gives
\[
 \dot w
 \le\frac2\mu-\alpha w
       +2\lVert\tilde{\xi}\rVert_\infty w.
\]
The last coefficient tends to zero. Scalar comparison therefore gives
constants $Z_R,Z_C>0$ such that, for all sufficiently large $t$,
\begin{equation}
 s(t)\le Z_R\mathcal{E}(t),\qquad \bar s(t)\le Z_C\mathcal{E}(t).               \label{eq:factor_totals}
\end{equation}

Lemma~\ref{lem:observability} also applies. Choose $t_1$ so that
\eqref{eq:observability} and \eqref{eq:factor_totals} hold for
$t\ge t_1$.

\smallskip\emph{A common window estimate.}
Let $d=mn$. For each $t\ge t_1$, choose $(i,j)$ such that
\begin{equation}
 G:=\int_t^{t+T}P_{ij}(u)^2\,du
 \ge \frac{c_0}{d}\mathcal{E}(t).
\end{equation}
For $t\le u\le t+T$, put
\[
 G(u)=\int_t^uP_{ij}(v)^2\,dv,
\]
so that $G(t)=0$ and $G(t+T)=G$.
Equations \eqref{eq:Rfilter} and \eqref{eq:Cfilter}, with their nonnegative
omitted terms, give
\[
 R_i(u)\ge\frac{e^{-\alpha T}}\mu G(u),
 \qquad
 C_j(u)\ge\frac{e^{-\alpha T}}\mu G(u).
\]
It follows that
\begin{align}
 \int_t^{t+T}R_i(u)C_j(u)P_{ij}(u)^2\,du
 &\ge \frac{e^{-2\alpha T}}{\mu^2}
       \int_t^{t+T}G(u)^2G'(u)\,du\nonumber\\
 &=\frac{e^{-2\alpha T}}{3\mu^2}G^3
 \ge c_1\mathcal{E}(t)^3,                                             \label{eq:cubic}
\end{align}
where $c_1>0$ is fixed.

\smallskip\emph{Upper bound when $\epsilon_{\mathrm{stab}}>0$.}
Boundedness of $R$ gives
$s(u)+\epsilon_{\mathrm{stab}}\le M$ for a fixed $M$. Integrating
\eqref{eq:edot}, retaining the single index chosen above, and using
\eqref{eq:cubic} yields
\begin{equation}
 \mathcal{E}(t)-\mathcal{E}(t+T)\ge c_2\mathcal{E}(t)^3.                                \label{eq:cubic_window}
\end{equation}

\smallskip\emph{Upper bound when $\epsilon_{\mathrm{stab}}=0$.}
The solution is not the equilibrium, and \eqref{eq:Rfilter} gives
$s(t)>0$ for every $t>0$. Throughout a window $[t,t+T]$ with
$t\ge t_1$, \eqref{eq:factor_totals} and monotonicity give
\[
 s(u)\le Z_R\mathcal{E}(u)\le Z_R\mathcal{E}(t).
\]
Equations \eqref{eq:edot} and \eqref{eq:cubic} now give
\begin{equation}
 \mathcal{E}(t)-\mathcal{E}(t+T)\ge c_3\mathcal{E}(t)^2.                                \label{eq:quadratic_window}
\end{equation}

\smallskip\emph{Iteration of the window estimates.}
Set $e_k=\mathcal{E}(t_1+kT)$. If some
$e_k=0$, then $\mathcal{E}\equiv0$ and the bound is trivial, so assume
$e_k>0$ for every $k$. If
\[
 e_k-e_{k+1}\ge c e_k^q,
\]
then the mean value theorem gives
\[
 e_{k+1}^{-(q-1)}-e_k^{-(q-1)}\ge(q-1)c.
\]
Thus $e_k=O(k^{-1/(q-1)})$. Monotonicity of $\mathcal{E}$ extends the bound from
the grid points to every $t\ge t_1$. Increasing the constant covers
$[0,t_1]$. Taking $q=3$ in
\eqref{eq:cubic_window} and $q=2$ in \eqref{eq:quadratic_window} gives
the two upper bounds.

\smallskip\emph{Lower bound when $\epsilon_{\mathrm{stab}}=0$.}
By \eqref{eq:xi_s_bound}, \eqref{eq:edot}, and
\eqref{eq:factor_totals},
\[
 -\dot{\mathcal{E}}\le sK\le2Z_R\mathcal{E}^2,
 \qquad\text{that is,}\qquad
 \frac{d}{dt}\frac{1}{\mathcal{E}}\le2Z_R,
\]
for $t\ge t_1$. Integrating from $t_1$ gives
\[
 \frac1{\mathcal{E}(t)}\le\frac1{\mathcal{E}(t_1)}+2Z_R(t-t_1),\qquad t\ge t_1,
\]
and hence a constant $b_1>0$ with $\mathcal{E}(t)\ge b_1/(1+t)$ for $t\ge t_1$.

\smallskip\emph{Lower bound when $\epsilon_{\mathrm{stab}}>0$.}
Since $R_i\le s$, $C_j\le\bar s$, and $K\le2\mathcal{E}$,
\eqref{eq:edot} and \eqref{eq:factor_totals} give
\[
 \tilde{\xi}_{ij}\le\frac{s\bar s}{\epsilon_{\mathrm{stab}}},
 \qquad
 -\dot{\mathcal{E}}\le\frac{s\bar s}{\epsilon_{\mathrm{stab}}}K
 \le\frac{2Z_RZ_C}{\epsilon_{\mathrm{stab}}}\mathcal{E}^3,
 \qquad\text{that is,}\qquad
 \frac{d}{dt}\frac1{\mathcal{E}^2}
 \le\frac{4Z_RZ_C}{\epsilon_{\mathrm{stab}}},
\]
for $t\ge t_1$. Integrating from $t_1$ gives
\[
 \frac1{\mathcal{E}(t)^2}\le\frac1{\mathcal{E}(t_1)^2}
 +\frac{4Z_RZ_C}{\epsilon_{\mathrm{stab}}}(t-t_1),\qquad t\ge t_1,
\]
and hence a constant $a_1>0$ with $\mathcal{E}(t)\ge a_1/\sqrt{1+t}$ for
$t\ge t_1$. Reducing $a_1$ and $b_1$, if necessary, extends both lower
bounds to $[0,t_1]$.
\end{proof}

\end{document}